\documentclass[lettersize,journal]{IEEEtran}
\usepackage{amsmath,amsfonts}
\usepackage{algorithmic}
\usepackage{algorithm}
\usepackage{array}
\usepackage[caption=false,font=normalsize,labelfont=sf,textfont=sf]{subfig}
\usepackage{textcomp}
\usepackage{stfloats}
\usepackage{url}
\usepackage{verbatim}
\usepackage{graphicx}
\usepackage{cite}

\usepackage{times}
\usepackage{soul}
\usepackage{url}
\usepackage[hidelinks]{hyperref}
\usepackage[utf8]{inputenc}
\usepackage[small]{caption}
\usepackage{amsthm}
\usepackage{booktabs}
\usepackage[switch]{lineno}

\usepackage{float}
\usepackage{subfig}
\usepackage{amsmath,amssymb,amsfonts}
\usepackage{xcolor}
\usepackage{booktabs}
\usepackage{color}
\usepackage{makecell}
\usepackage{multirow}

\newtheorem{theorem}{Theorem}
\newtheorem{assumption}{Assumption}
\newtheorem{lemma}{Lemma}
\newtheorem{remark}{Remark}

\begin{document}

\title{The Diversity Bonus: Learning from Dissimilar Distributed Clients in Personalized Federated Learning}

\author{Xinghao~Wu, Xuefeng~Liu, Jianwei~Niu,~\IEEEmembership{Senior~Member,~IEEE}, Guogang~Zhu, Shaojie~Tang, Xiaotian Li, and Jiannong Cao,~\IEEEmembership{Fellow, IEEE}
		\IEEEcompsocitemizethanks{\IEEEcompsocthanksitem X. Wu, X. Liu, J. Niu and G. Zhu are with the School of Computer Science and Engineering, Beihang University, China.
			E-mail: wuxinghao@buaa.edu.cn, liu\_xuefeng@buaa.edu.cn, niujianwei@buaa.edu.cn, buaa\_zgg@buaa.edu.cn.
			\IEEEcompsocthanksitem S. Tang is with the Jindal School of Management, The University of Texas at Dallas. E-mail: tangshaojie@gmail.com. 
               \IEEEcompsocthanksitem X. Li is with the Department of Statistics, The University of Chicago. E-mail: xiaotian13@uchicago.edu. 
               \IEEEcompsocthanksitem J. Cao is with the Department of Computing, The Hong Kong Polytechnic University, Hong Kong. E-mail: jiannong.cao@polyu.edu.hk.
			\IEEEcompsocthanksitem Corresponding author: Jianwei Niu.
	}}

\markboth{Journal of \LaTeX\ Class Files,~Vol.~14, No.~8, August~2021}%
{Shell \MakeLowercase{\textit{et al.}}: A Sample Article Using IEEEtran.cls for IEEE Journals}


\maketitle

\begin{abstract}
Personalized Federated Learning (PFL) is a commonly used framework that allows distributed clients to collaboratively and in parallel train their personalized models. PFL is particularly useful for handling situations where data from different clients are not independent and identically distributed (non-IID). Previous research in PFL implicitly assumes that clients can gain more benefits from those with similar data distributions. Correspondingly, methods such as personalized weight aggregation are developed to assign higher weights to similar clients during training. We pose a question: can a client benefit from other clients with dissimilar data distributions and if so, how? This question is particularly relevant in scenarios with a high degree of non-IID, where clients have widely different data distributions, and learning from only similar clients will lose knowledge from many other clients. We note that when dealing with clients with similar data distributions, methods such as personalized weight aggregation tend to enforce their models to be close in the parameter space. It is reasonable to conjecture that a client can benefit from dissimilar clients if we allow their models to depart from each other. Based on this idea, we propose DiversiFed which allows each client to learn from clients with diversified data distribution in personalized federated learning. DiversiFed pushes personalized models of clients with dissimilar data distributions apart in the parameter space while pulling together those with similar distributions. In addition, to achieve the above effect without using prior knowledge of data distribution, we design a loss function that leverages the model similarity to determine the degree of attraction and repulsion between any two models. Experiments on three benchmark datasets and a public medical dataset show that DiversiFed can benefit from dissimilar clients and thus outperform the state-of-the-art methods, especially when the non-IID degree is high. 
\end{abstract}


\section{Introduction}\label{introduction}
\IEEEPARstart{I}{n} recent years, the growth of AI applications and the data utilized to train AI models has been rapid, but it faces challenges due to the distribution of this data across organizations, edge devices, IoT devices, and more. Privacy protection laws like the EU's General Data Protection Regulation (GDPR) \cite{regulation2016regulation} prevent this scattered data from being gathered in data centers for training large-scale AI models. This is a notable obstacle in areas like the medical field, where data from various institutions can't be centralized to train robust models, limiting the full potential of AI. A solution to this issue is Federated Learning (FL) \cite{mcmahan2017federated}, a privacy-preserving distributed machine learning framework. It allows distributed entities to collaboratively and in parallel train a global model that performs well on all participant data, without needing to share the raw data.

A great challenge faced by FL is the non-independent and identically distribution (non-IID) of data among clients. For instance, different user preferences result in varied distributions of photo categories stored on separate mobile devices. In the presence of non-IID data, the performance of a global model trained by traditional FL is greatly degraded \cite{li2020federated, karimireddy2020scaffold}. 
One approach to address this non-IID problem is personalized federated learning (PFL) \cite{tan2021towards}, in which each client trains a personalized model to better suit its local data distribution. A crucial aspect of PFL is determining how to obtain knowledge from other clients to enhance the generalization of personalized models.

A widely used strategy in PFL is that each client maintains a global model and a personalized model simultaneously. The global model is trained by all clients and the personalized model learns from the global model through fine-tuning \cite{yang2020fedsteg,fallah2020personalized,acar2021debiasing}, model regularization \cite{t2020personalized,li2021ditto}, etc. As the global model contains all clients' knowledge, these methods can allow each client to obtain other clients' help. In these approaches, the global model is obtained by aggregating all client models with the same weight. Consequently, a client receives the same level of assistance from all clients when learning from the global model. More recently, some studies, such as FedAMP \cite{huang2021personalized} and APPLE \cite{ijcai2022p301}, find that \textit{if a client can get more help from those clients with similar data distribution, its personalized model can achieve better performance}. Therefore, they design personalized weight aggregation methods in which a personalized model receives more contribution from clients with similar data distribution.

We observe that in the above approach, a client mainly gains help from clients with similar data distribution. This also implies that learning from clients with dissimilar distribution has a negative effect. Hence, we raise the following question: \textit{Is it possible a client can still benefit from those with dissimilar data distributions?} 

The answer to the question above can be gained by examining why a client can benefit from those with similar data distributions. In the personalized weight aggregation methods, the essence is to bring the personalized models of clients with similar data distribution closer in the parameter space. Since the optimal personalized models of clients with similar data distribution should be located close to each other in the parameter space, keeping these models close to each other during training can provide more benefits. Similarly, for clients with dissimilar data distributions, since their ultimate personalized models are expected to be different, a client should also benefit by \textit{keeping its personalized model away from those with dissimilar data distributions during the training process}. 
 
\begin{figure}[t]
		\centerline{\includegraphics[width=\linewidth]{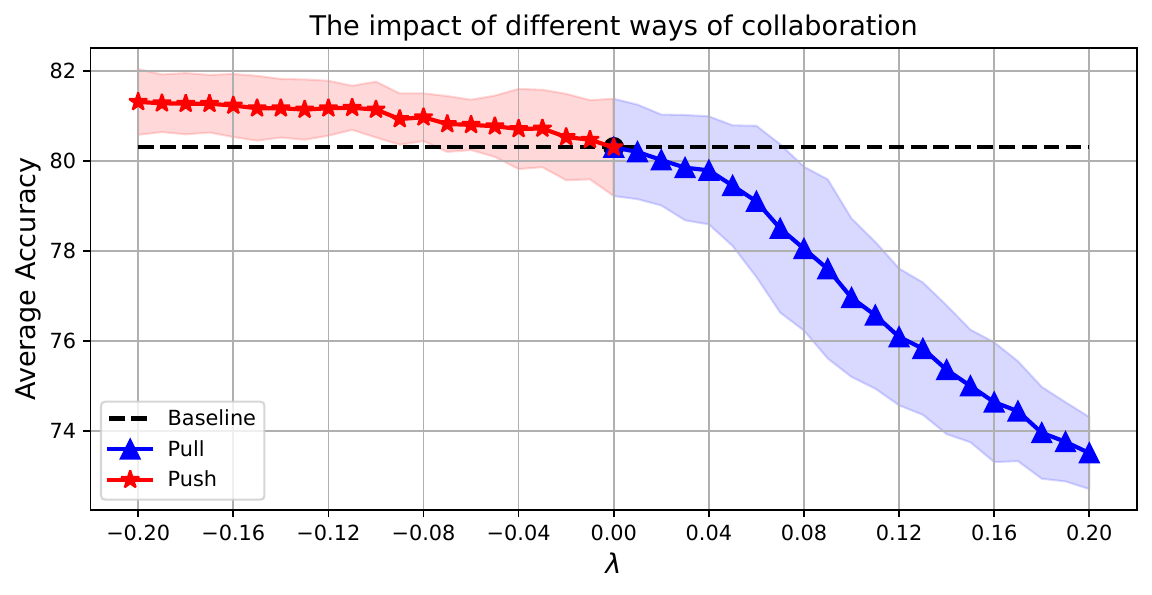}}
        \caption{An experiment to verify the effect of pulling personalized models together and pushing personalized models apart when the data distribution is very different.}
		\label{toy example result}
\end{figure}
To demonstrate the above conjecture, we carry out a preliminary experiment in a PFL scenario when the non-IID degree is high. Specifically, we set up five clients on the CIFAR-10 dataset, each with two classes of data without any class overlap. For a certain client $i$, the objective function of its personalized model $w_i$, here denoted as $L(w_i)$, is defined as
\begin{linenomath}
\begin{equation}\label{eq:pull and push}
	L(w_i) = L_{ce}(w_i) + \lambda\sum_{j \ne i}||w_i - w_j||_2,
\end{equation}
\end{linenomath}
where $L_{ce}$ represents the cross-entropy loss on its local data, and $\lambda$ is a hyperparameter used to control the distances among different clients. We can see when $\lambda > 0$, the personalized models of all five clients are ``pulled together" during the training process, and when $\lambda<0$,  they are ``pushed apart". We investigate the impact of various degrees of ``pulling together" and ``pushing apart" on the performance of PFL. For each $\lambda$ ranging from $-0.2\sim 0.2$, we repeat the experiments 10 times, and the average accuracy and \textcolor{black}{the standard deviation are} shown in Fig.~\ref{toy example result}. 

The black line in the graph represents the scenario of $\lambda=0$, in which none of the personalized models learn from each other, serving as a baseline. First, we can see that when $\lambda >0$, the average accuracy is lower than the baseline. This suggests that in the current scenario with a high degree of non-IID, even a small degree of ``pulling together" can have negative effects. This observation aligns well with the FedAMP and APPLE, where a client gets little help from those with dissimilar data distributions. More importantly, when $\lambda < 0$, the average accuracy is higher than the baseline. This confirms our conjecture above that personalized models can benefit from dissimilar clients by pushing them apart.

\begin{figure}[b]
 \centering
		\centerline{\includegraphics[width=\linewidth]{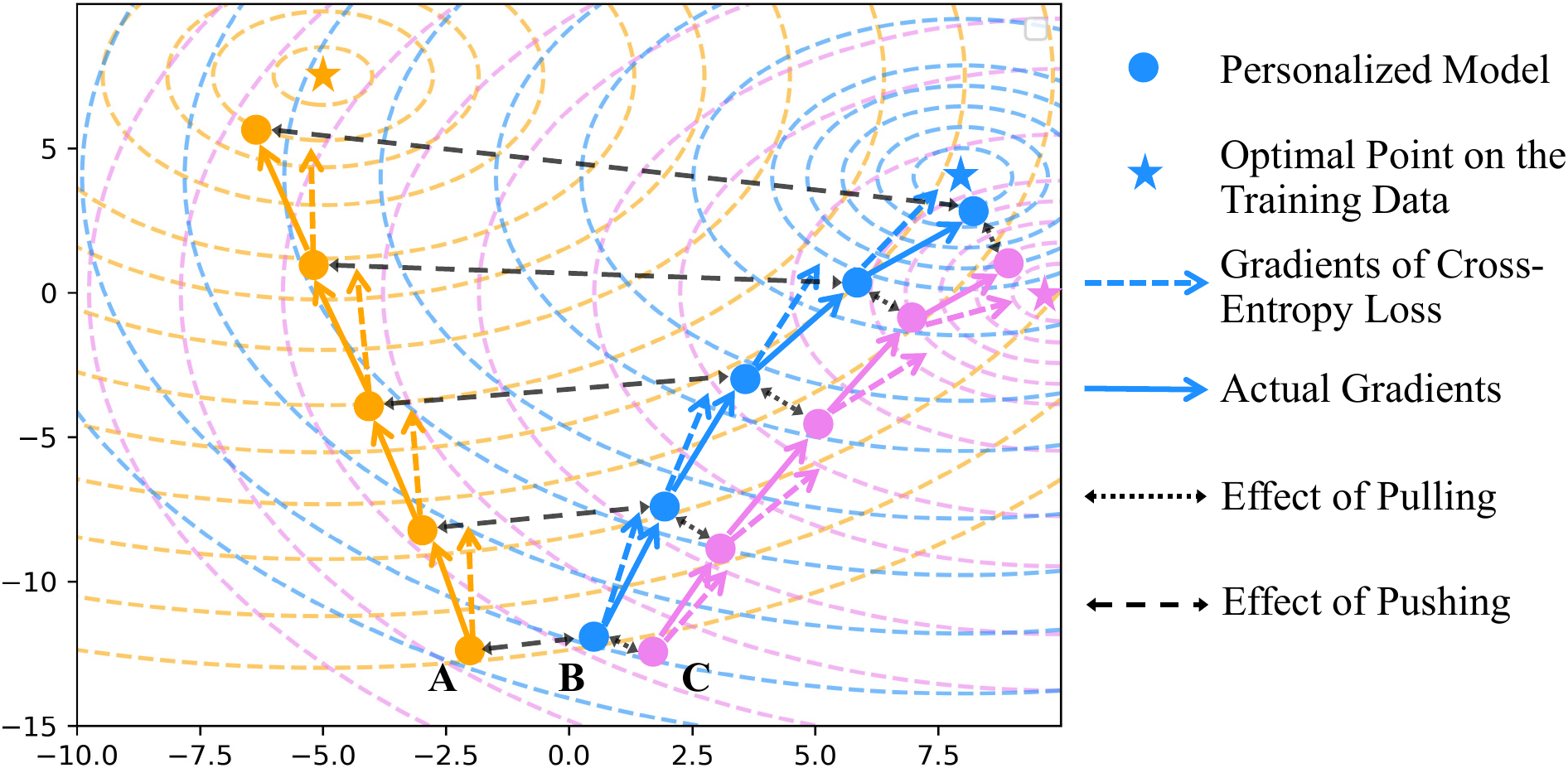}}
 \caption{A toy example to show the training process of DiversiFed. Similar clients (i.e., B and C) are pulled together and the dissimilar client (i.e., A) is pushed apart.}
		\label{toy example}
\end{figure}

Eq.~\eqref{eq:pull and push} provides a rudimentary strategy for regulating a client's assistance from others in PFL by modifying the $\lambda$. Nonetheless, this approach faces several obstacles. Firstly, privacy constraints in FL prevent advanced knowledge of client data distributions, making it challenging to determine which models should be attracted or repulsed. Secondly, the method becomes impractical when dealing with a large number of clients, as adjusting $\lambda$ individually is not feasible. Lastly, this technique necessitates transferring all clients' models to each client, leading to communication overhead and privacy concerns, rendering the approach impracticable.

In this paper, to overcome these challenges and implement our ideas, we introduce a novel method named DiversiFed that allows each client to learn from not only similar clients but also dissimilar clients. This method includes a new loss function based on model distance, which features both attractive and repulsive effects between any two models. Within the FL group, this function assesses the intensity of these effects based on model similarities. As a result, personalized models with similar data distributions are drawn closer, and those with dissimilar distributions are pushed apart without needing any prior knowledge of data distribution. When optimizing the loss function, we further propose to utilize an incremental proximal optimization framework to reduce communication costs and privacy issues.

Fig.~\ref{toy example} shows a toy example of the training process in DiversiFed of three clients, $A$, $B$, and $C$. The distribution of client $A$ is significantly different from the other two clients $B$ and $C$, as its optimal personalized model, denoted as yellow $\star$, is far from the other two. Clients $B$ and $C$ have similar optimal personalized models. Take the training process of client $B$ as an example. We can see that at each update, the actual gradients (shown as solid blue arrows) are affected not only by the cross-entropy loss (shown as dashed blue arrows) but also by the model distance loss. This causes its model to be closer to the model of client $C$ and farther away from client $A$.

Our main contributions can be summarized as follows:
\begin{itemize}
	\item We observe that the current PFL methods do not effectively leverage the knowledge of clients with dissimilar data distributions. To effectively utilize this knowledge, we demonstrate that it would be beneficial to push their models apart in the parameter space.
	\item We propose a new PFL algorithm named DiversiFed. By introducing the loss function based on the model distance in the parameter space, personalized models with similar data distributions are pulled together, while those with dissimilar data distributions are pushed apart. 
	\item We verify our method under multiple non-IID settings on three benchmark datasets and a public medical dataset. The experimental results demonstrate that our method outperforms state-of-the-art methods.
\end{itemize}

\section{Related Work}
\subsection{Traditional Federated Learning}
Traditional FL involves a collective effort from all distributed clients in parallel to train a global model under the server's coordination. However, in practical application scenarios, client data distributions within FL are often non-IID, leading to diverse local training objectives that can significantly hamper global model performance. Addressing the non-IID problem has become a paramount concern, posing substantial challenges for FL applications.

Predominant research has primarily focused on diminishing the impact of non-IID data on the global model by addressing client drift. One method, FedProx \cite{li2020federated}, introduces constraints on local model updates through an additional $L_2$ regularization term, ensuring the local models remain proximal to the global model throughout the update process. In contrast, SCAFFOLD \cite{karimireddy2020scaffold} implements a gradient correction term that accumulates the gradient deviations from the global model during training. This correction is subsequently used to adjust local updates, thereby aligning them closer to the global model. FedNTD \cite{lee2022preservation} employs knowledge distillation, enabling clients to absorb knowledge from the global model during local updates, which can be seen as a softer version of FedProx.

These methods have proven effective in minimizing the impact of non-IID data and enhancing the global model's accuracy. However, a limitation of traditional FL lies in its reliance on a single global model that has to cater to the data distributions of all participating clients. This approach may falter when trying to optimize performance on specific clients' data, given the global model's intent to accommodate the diverse data distributions of all clients.

\subsection{Personalized Federated Learning}
The limitation of the traditional FL framework leads to the exploration of personalized federated learning (PFL) approaches, which aim to address the challenge of client-specific performance within the FL framework. In PFL, instead of training a single global model, personalized models are created for individual clients, tailoring the learning process to their specific data characteristics. By adopting personalized models, PFL approaches can provide improved performance on each client's data while still benefiting from the collaborative nature of FL. This allows for a more fine-grained adaptation to the local data distributions, ultimately leading to enhanced accuracy and efficiency in FL scenarios. PFL has emerged as a prominent research direction within the FL field, with several categories of approaches being explored. The current mainstream PFL work can be broadly categorized into the following:

\textbf{Meta-learning-based methods.} These methods adopt the meta-learning concept where all clients cooperatively train a meta-model, which is subsequently fine-tuned by each client to yield personalized models. Notable PFL techniques such as Per-FedAvg \cite{fallah2020personalized}, FedMeta \cite{chen2018federated}, and others \cite{jiang2019improving} incorporate Model-Agnostic Meta-Learning (MAML) to attain this. However, it has been argued in \cite{acar2021debiasing} that prior meta-learning-based methods suffer from discrepancies between the local optimization objectives of clients and the global optimization goal. In response, they proposed a debiased technique that modifies the client's loss function, aiming to eradicate biases for more accurate training.

\textbf{Model-regularization-based methods.} These PFL strategies involve clients collectively creating a global model initially. During the personalized model training on the client side, a regularization term is introduced to deter the personalized models from straying too far from the global model. L2SGD \cite{hanzely2020federated}, pFedMe \cite{t2020personalized}, and Ditto \cite{li2021ditto} are notable methods employing this approach. They share a similar concept in training personalized models but exhibit differences in terms of their proposed global model construction and optimization strategies.

\textbf{Parameter-decoupling-based method.} This approach divides the neural network into two segments at the layer level: one shared amongst all clients, and the other personalized. Proposals like FedPer \cite{arivazhagan2019federated} and FedRep \cite{collins2021exploiting} aim to personalize the classifier while sharing the feature extractor. FedBN \cite{li2021fedbn} and MTFL \cite{mills2021multi} focus on personalizing the BN layers of the model. These methods manually and empirically select personalized layers. Some works try to pick layers automatically, such as through reinforcement learning \cite{sun2021partialfed} and hypernetwork \cite{ma2022layer}.

\textbf{Knowledge-distillation-based methods.} Knowledge distillation (KD) offers a novel approach to client collaboration. Unlike traditional methods that share knowledge by aggregating model parameters, KD-based methods allow clients to learn from the soft labels generated by other models. In FML \cite{shen2020federated}, clients use the global model as a teaching model and learn from it through KD during local updates. pFedSD \cite{jin2022personalized} uses KD to distill knowledge from the personalized models in the previous round during the client's local update to mitigate the catastrophic forgetting caused by model aggregation. FedProto \cite{tan2022fedproto}, grounded in prototype learning, calculates the prototype of the logit for each class of all clients and distributes these prototypes. Clients learn from these prototypes through KD.

Viewed from a knowledge transfer perspective, the aforementioned methods involve clients receiving assistance from other clients through collaboratively trained global models. All clients contribute equally to the global model as they hold the same aggregate weight, meaning each client receives an equal level of help from others. However, recent research posits that a client can garner more benefits from getting more assistance from clients with similar data distributions. This idea has led to the emergence of personalized-weight-aggregation-based methods as a popular research topic.

\textbf{Personalized-weight-aggregation-based methods.} Under this strategy, FedAMP \cite{huang2021personalized} designs an attention-inducing function to draw more similar personalized models closer in the parameter space. FedFomo \cite{zhang2020personalized} measures the model performance of different clients using local validation sets and steers the personalized model closer to those performing better. APPLE \cite{ijcai2022p301} guides clients to learn from those with similar data distributions by learning a directed relationship vector. CGPFL \cite{ijcai2022p311} builds multiple generalized models at the context level, pushing the personalized model towards the most relevant generalized model to learn from clients with similar data distributions. SFL \cite{ijcai2022p357} leverages graph-based structure information to build relationships between clients, fostering more collaboration between neighboring clients. However, these methods primarily focus on clients with similar data distributions. For those with dissimilar data distributions, these methods either inhibit their personalized models from collaborating or make minimal contributions to each other. As a result, personalized models of clients struggle to obtain help from those with dissimilar data distributions.

\section{Learning from Both Similar and Dissimilar Clients in Personalized Federated Learning}
In this section, we first define the problem of PFL. Then we demonstrate how to design model distance loss on the parameter space to pull similar models together while pushing dissimilar models apart. Lastly, we use an incremental proximal optimization framework to address the optimization problem that we have outlined.

\subsection{Overview of DiversiFed}\label{overview}

\begin{figure}[tb]
		\centerline{\includegraphics[width=\linewidth]{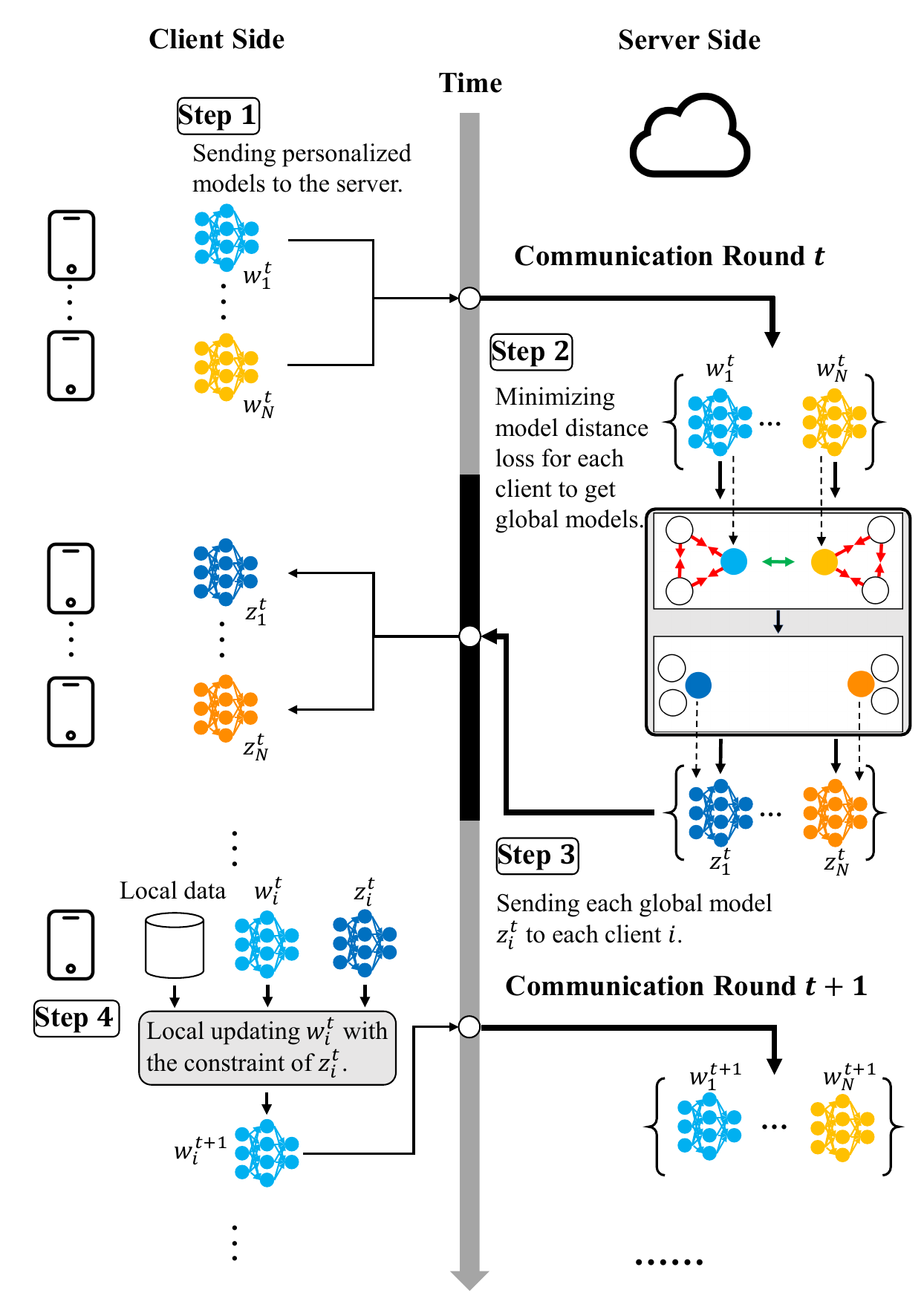}}
 \caption{The system workflow of DiversiFed.}
		\label{overview of DiversiFed}
\end{figure}

As shown in Fig.~\ref{overview of DiversiFed}, the training process of DiversiFed in each communication round can be summarized as follows:

Initially, each client uploads their personalized model $w_i$ in parallel to the server. Subsequently, the server minimizes the model distance loss, aiming to pull similar models together while pushing dissimilar models apart within the parameter space, thereby generating a global model $z_i$ for each client. Following this, the server dispatches each global model to the corresponding client. Ultimately, during local updates, each client maintains their personalized model in approximation to the global model to get help from other clients.

In the below sections, we will illustrate the details of DiversiFed step by step.

\subsection{PFL Problem Definition}\label{Problem Definition}
In PFL, the objective differs from traditional FL as it aims to ensure that each client's personalized model performs well on its respective local data distribution. We can represent the loss function of client $i$ as $L_i(w_i; D_i): \mathcal{W} \mapsto \mathbb{R}$, where $w_i$ denotes the personalized model of client $i$ and $D_i$ represents the data distribution of client $i$. The optimization objective of PFL can then be formulated as:
\begin{equation}
\min_{w_1,w_2,...,w_N} \sum_{i=1}^{N} L_i(w_i; D_i),
\end{equation}
where $N$ is the total number of distributed clients in the FL system. A straightforward approach to optimize this objective is for each client to train its personalized model $w_i$ using its local data, minimizing the loss function (e.g., cross-entropy loss). However, due to the limited amount of local data available in FL, the personalized model $w_i$ is easy to overfit the training data distribution $D_i^{train}$, resulting in poor generalization on the real data distribution $D_i$. Additionally, in non-IID scenarios, the data distributions among clients can be dissimilar (i.e., $D_i \neq D_j$ for $i \neq j$). Consequently, designing an appropriate loss function that allows a client's personalized model to learn effectively from clients with diverse data distributions in PFL poses a significant challenge.

\subsection{Designing the Loss Function in DiversiFed}\label{Introduce Contrastive Loss to PFL}
As we discussed in the section \ref{introduction}, during the training process, each client's model needs to learn the local data distribution while approaching or departing from other clients' models to obtain help. Therefore, we formulate the training objective for each client as
\begin{linenomath}
\begin{equation}\label{loss function}
	L_i(w_i; D_i) = L_{e}(w_i; D_i^{train}) + \lambda  L_{d}(w_i),
\end{equation}
\end{linenomath}
where $L_{e}$ denotes the empirical loss (e.g., cross-entropy loss in the classification task) on the local training data $D_i^{train}$, while $L_{d}$ is the model distance loss used to pull or push models in the parameter space. $\lambda$ is a hyperparameter to control the effect of $L_{d}$ in the overall training objective.

Designing $L_{d}$ faces two challenges. One is how to know the data distribution similarity among clients in the training process. The second is how to decide which clients should approach and which clients should depart, and to what extent. We tackle the first challenge by following standard practice in FL research, using model similarity to indirectly signify data distribution similarity. For the second challenge, we propose a method that embodies both attraction and repulsion effects between any two models. Within the FL group, more similar models experience a greater attraction force than the repulsion force. Conversely, if the models are more dissimilar, the repulsion force is dominant. This mechanism ensures similar models come closer and dissimilar ones distance themselves.

To implement this concept, we specify the $L_{d}$ in DiversiFed as
\begin{linenomath}
\begin{equation}\label{contrastive loss}
	L_{d}(w_i) = \frac{1}{|a(i)|} \cdot \sum_{j \in a(i)} \log \left(\frac{\exp(\frac{||w_i - w_j||}{\tau})}{\sum_{j \in a(i)} \exp(\frac{||w_i - w_j||}{\tau})} \right).
\end{equation}
\end{linenomath}
Here, $a(i)$ signifies all clients excluding client $i$ with $|a(i)|=N-1$, and $||w_i - w_j||$ denotes the Euclidean distance within the parameter space between client $i$'s and $j$'s personalized models. $\tau$ is a temperature coefficient used to scale the model distance.

When minimizing the $L_{d}$, for each $w_j$, $w_i$ is attracted to it (i.e., minimizing the numerator in the $\log$ function) and pushed away by the other models (i.e., maximizing the denominator in the $\log$ function). Ultimately, the composite of all attraction and repulsion effects dictates which personalized models $w_i$ should approach or deviate from. To simplify comprehension, we present a toy example with three personalized models.

\begin{figure}[b]
	\centering
	\subfloat[]{
		\label{effect of item 1}
		\includegraphics[width=0.31\linewidth]{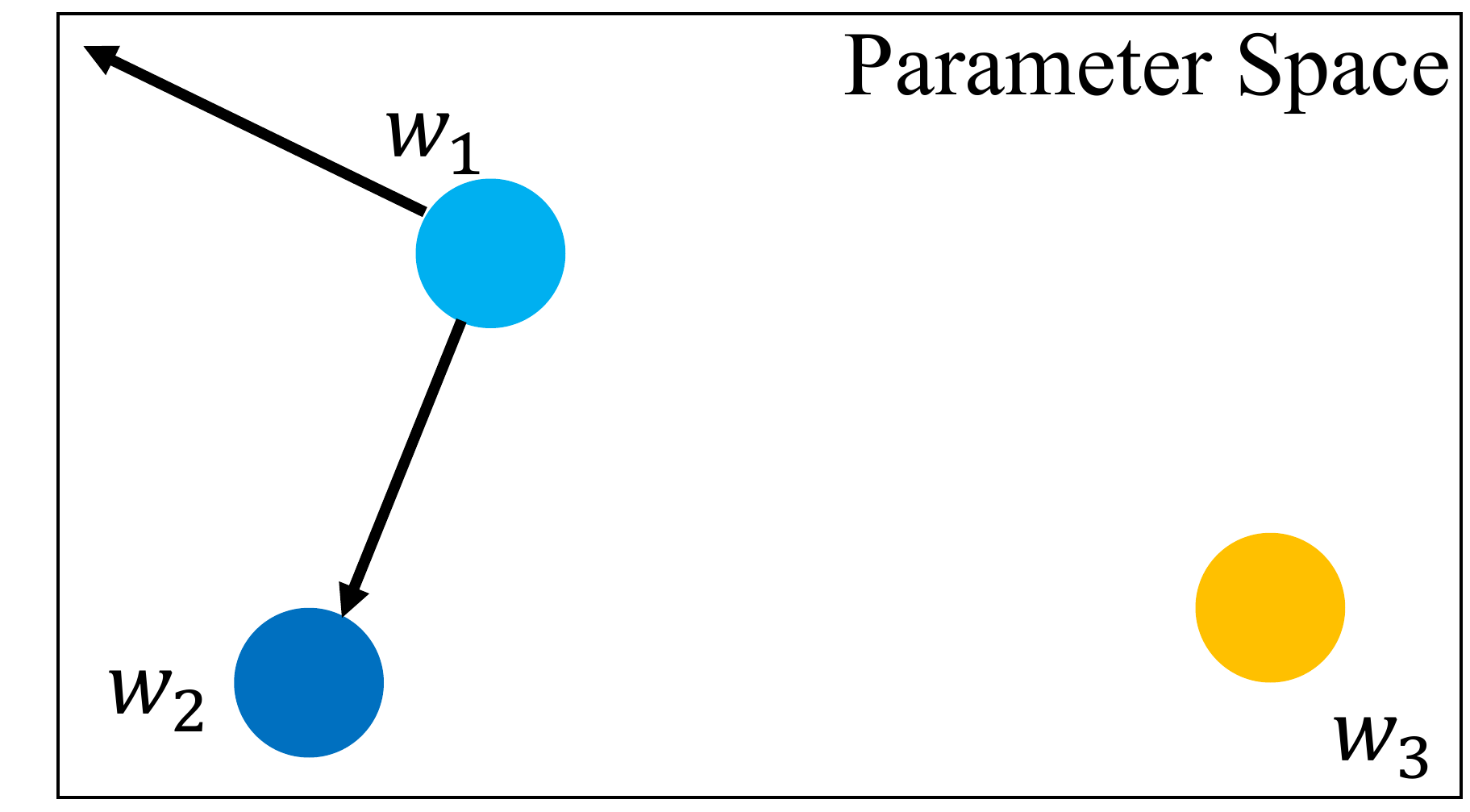}}
	\subfloat[]{
		\label{effect of item 2}
		\includegraphics[width=0.31\linewidth]{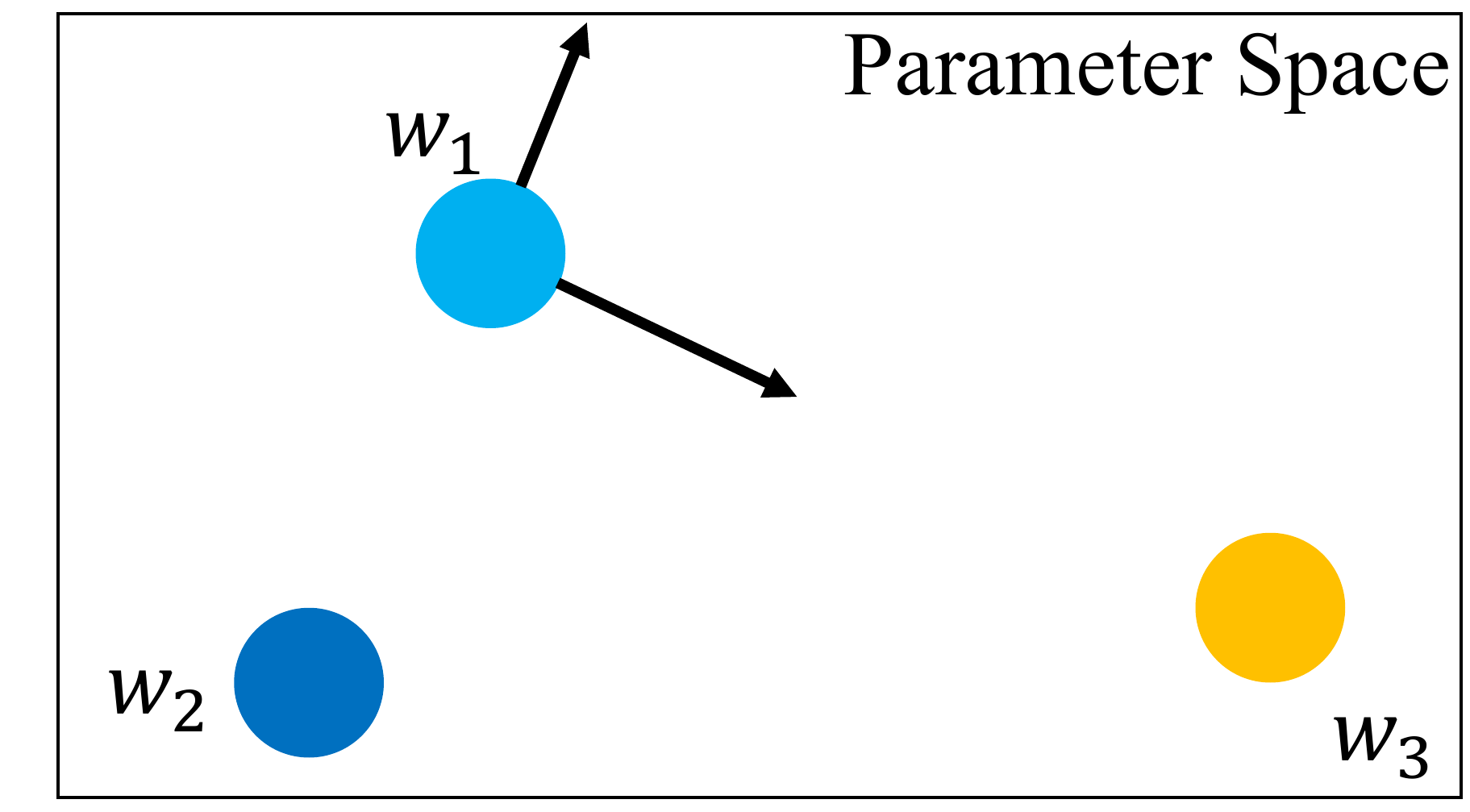}}
	\subfloat[]{
		\label{effect of item 1 + item 2}
		\includegraphics[width=0.31\linewidth]{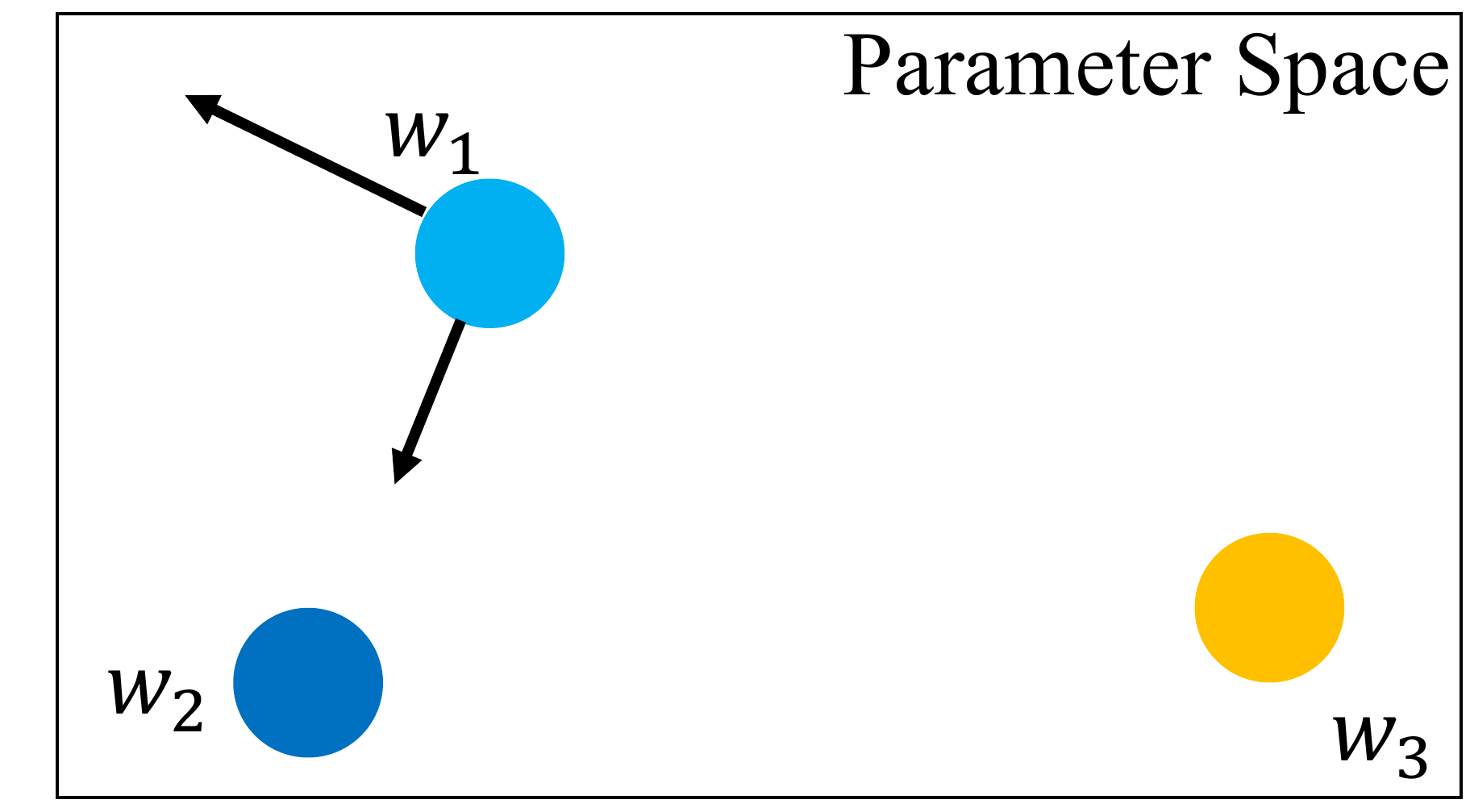}}
	
	\caption{A toy example to show the effect of $L_{d}$ in Eq.~\eqref{example of cl loss}. (a) shows the effect of item 1. (b) shows the effect of item 2. (c) shows the combined effect of item 1 and item 2.}
	\label{fig:example of cl loss}
\end{figure}

As shown in Fig.~\ref{fig:example of cl loss}, there are three personalized models $w_1$, $w_2$, and $w_3$ in the parameter space. $w_1$ and $w_2 $ are closer than $w_1 $ and $w_3$. By Eq~\eqref{contrastive loss}, the model distance loss of $w_1$ is
\begin{linenomath}
\begin{equation}\label{example of cl loss}
L_{d}(w_1) = \underbrace{ \frac{1}{2} \log \left(\text{softmax}(d_2) \right) } _\text{Item 1} + \underbrace{ \frac{1}{2} \log \left(\text{softmax}(d_3) \right) }_\text{Item 2},
\end{equation}
\end{linenomath}
where $d_k = \frac{||w_1-w_k||}{\tau}$ and $\text{softmax}(d_k) = \frac{\exp(d_k)}{\sum_{j \in [2,3]}\exp(d_j)}$, $k=2, 3$.
The effect of item 1 in $L_{d}(w_1)$ is shown in Fig.~\ref{effect of item 1}, where $w_1$ is pulled closer by $w_2$ and pushed away by $w_3$. Similarly, the effect of item 2 in $L_{d}(w_1)$ is shown in Fig.~\ref{effect of item 2}, where $w_1$ is pushed away by $w_2$ and pulled closer by $w_3$. As we will analyze theoretically in Lemma \ref{lemma1}, since $w_1$ and $w_2$ are more similar than $w_1$ and $w_3$ in parameter space, the pulling together effect between $w_1$ and $w_2$ in Fig.~\ref{effect of item 1} is higher than the pushing apart effect in Fig.~\ref{effect of item 2}. The pulling together effect between $w_1$ and $w_3$ in Fig.~\ref{effect of item 1} is lower than the pushing apart effect in Fig.~\ref{effect of item 2}. The combined effect of item 1 and item 2 is shown in Fig.~\ref{effect of item 1 + item 2}, that is, $w_1$ is approaching $w_2$ and departing from $w_3$.

\subsection{An Incremental Proximal Method to Minimize the Loss Function}\label{Incremental Optimization Method}
A straightforward way to optimize Eq.~\eqref{contrastive loss} is that each client optimizes locally using gradient descent. However, this approach requires the server to transmit the personalized models $w_j$ of all distributed clients to each individual client $i$. This leads to two issues. Firstly, in scenarios with a large number of clients, the communication overhead can become significant. Secondly, sharing other clients' models raises privacy concerns and the risk of information leakage. To address these challenges, we propose adopting a commonly used incremental proximal optimization framework based on the approach described in \cite{bertsekas2011incremental}. This framework allows for the step-by-step optimization of the objective function on both the server-side and the client-side.

At each communication round $t$, each client uploads its personalized model $w_i^t$ to the server. The server executes one step gradient descent 
\begin{linenomath}
\begin{equation}\label{server optimize}
	z_i^{t} = w_i^{t} - \alpha_t \nabla_{w_i^{t}} L_{d}(w_i^{t})
\end{equation}
\end{linenomath}
to minimize $L_{d}$, where $\alpha_t$ is the learning rate. On the client-side, after receiving $z_i^{t}$ from the server, each client optimizes 
\begin{linenomath}
\begin{equation}\label{client optimize}
	w_i^{t+1} = \arg \min_{w_i} \{L_{e}(w_i; D_i^{train}) + \frac{\lambda}{2\alpha_t}||w_i - z_i^{t}||^2\}
\end{equation}
\end{linenomath}
to get the new personalized model $w_i^{t+1}$. The above process is iterated until convergence or a preset maximum communication round $T$ is reached. The details of the training process are summarized as pseudocode in Algorithm \ref{alg:DiversiFed}.

\begin{algorithm}[tb]
	\caption{DiversiFed}
	\label{alg:DiversiFed}
	{\small
		\begin{algorithmic}
			\STATE {\bfseries Input:} Each client's initial personalized model $w_i^0$;
			Number of distributed clients $N$;
			Total communication rounds $T$;
			Hyperparameters $\lambda$ and $\tau$;
			Learning rate $\alpha_t$.
			\STATE {\bfseries Output:} Personalized model $w_i^T$ for each client. \\
			\FOR {$t = 0$ to $T-1$}
			\STATE Server sends $z_i^t$ to each client.
			\FOR{$i=1$ to $N$ \textbf{in parallel}}
			\IF{$t==0$}
			    \STATE Client $i$ optimizes $\arg \min_{w_i} \{L_{e}(w_i; D_i^{train}) \}$ for several local epochs to obtain $w_i^{t+1}$.
			\ELSE
			    \STATE Client $i$ optimizes $\arg \min_{w_i} \{L_{e}(w_i; D_i^{train})+ \frac{\lambda}{2\alpha_t}||w_i - z_i^{t}||^2 \}$ for several local epochs to obtain $w_i^{t+1}$.
			\ENDIF
			\STATE Client $i$ sends $w_i^{t+1}$ to the server.
			\ENDFOR
			\STATE Server executes $z_i^{t+1} = w_i^{t+1} - \alpha^{t+1} \nabla_{w_i^{t+1}} L_{d}(w_i^{t+1})$ to get a $z_i^{t+1}$ for each client.
			\ENDFOR
	\end{algorithmic}}
\end{algorithm}

Based on this framework, in each communication round, each client $i$ only uploads the personalized model $w_i^t$ to the server and the server only needs to deliver a global $z_i^t$ to client $i$. This approach ensures that the communication overhead of DiversiFed is comparable to traditional FL methods such as FedAvg in each round. \textcolor{black}{Also, as we prove in Lemma \ref{lemma1}, the $z_i^t$ is a linear combination of the personalized models $w_j^t, j \in [1, N]$. Client $i$ can not infer the specific personalized models of other clients. As a result, privacy is protected.}

\begin{lemma}\label{lemma1} \label{lemma1}
    If we optimize objective \eqref{loss function} by Eq.~\eqref{server optimize} and Eq.~\eqref{client optimize}, let $d_j=\frac{||w_i - w_j||}{\tau}$, we have
    \begin{enumerate}
    \item when $\frac{1}{|a(i)|} > \frac{\exp(d_j)}{\sum_{j \in a(i)}\exp(d_j)} $, $w_i$ approaches $w_j$;
    \item when $\frac{1}{|a(i)|} < \frac{\exp(d_j)}{\sum_{j \in a(i)}\exp(d_j)}$, $w_i$ departs from $w_j$. 
\end{enumerate}
\end{lemma}

\begin{proof}
According to Eq.~\eqref{server optimize}, we can obtain
\begin{equation}
\begin{aligned}
z_i^{t} &= w_i^t - \alpha_t \nabla_{w_i^t} L_{d}(w_i^t) \\
&=w_i^t - \alpha_t \sum_{j \in a(i)} \left( \frac{1}{|a(i)|} - \frac{\exp(d_j^t)}{\sum_{j \in a(i)}\exp(d_j^t)} \right)  \frac{w_i^t}{\tau^2 d_j^t} \\
&+ \alpha_t \sum_{j \in a(i)} \left(\frac{1}{|a(i)|} - \frac{\exp(d_j^t)}{\sum_{j \in a(i)} \exp(d_j^t)} \right) \frac{w_j^t}{\tau^2 d_j^t}.
\end{aligned}
\end{equation}
Let $\xi_j^t = \frac{1}{|a(i)|} - \frac{\exp(d_j^t)}{\sum_{j \in a(i)}\exp(d_j^t)}$ , we have
\begin{equation}
\begin{aligned}
z_i^t &= \left(1 - \alpha_t \sum_{j \in a(i)} \xi_j^t \cdot  \frac{1}{\tau^2 d_j^t} \right)  w_i^t +\alpha_t \sum_{j \in a(i)}  \xi_j^t \cdot \frac{1}{\tau^2 d_j^t}   w_j^t  \\
&= \beta_i^t w_i^t + \sum_{j \in a(i)}\beta_j^t w_j^t,
\end{aligned}
\end{equation}
where $\beta_i^t + \sum_j \beta_j^t=1$. That is, $z_i^t$ is a linear combination of all personalized models. Let $\alpha_t=1$, so we can obtain
\begin{equation}\label{omega j}
	\begin{aligned}
		\beta_j^t &= \left(\frac{1}{|a(i)|} - \frac{\exp(d_j^t)}{\sum_{j \in a(i)}\exp(d_j^t)} \right)\frac{1}{\tau^2 d_j^t} \\
		&= \left(\frac{1}{N-1} - \text{softmax}(d_j^t) \right)\frac{1}{\tau^2 d_j^t}.
	\end{aligned}
\end{equation}
Obviously, $\frac{1}{\tau^2 d_j^t}>0$. Therefore,  $\beta_j^t > 0$ when $\frac{1}{|a(i)|} - \frac{\exp(d_j^t)}{\sum_{j \in a(i)}\exp(d_j^t)}>0$, and vise versa. Meanwhile, according to Eq.~\eqref{client optimize}, when $\beta_j^t>0$, $w_i^t$ should approach $w_j^t$. When $\beta_j^t<0$, $w_i^t$ should depart from $w_j^t$. This proves Lemma 1.
\end{proof}
In the proof, it is demonstrated that the global parameter $z_i^t$ is a linear combination of the personalized model $w_i^t$. This property ensures that each client $i$ cannot directly infer the specific personalized models of other clients, such as $w_j^t$ for $j \in [1, N]$, from the received global parameter $z_i^t$. This protects the privacy. We can also see that the more similar the two models are, the smaller $d_j^t$ is and the larger $\beta_j^t$ is. $w_i^t$ should be closer to $w_j^t$.

\begin{remark}
    By Lemma \ref{lemma1}, we can know that if the personalized models of two clients are significantly different, then $\text{softmax}(d_j) > \frac{1}{N-1}$ and $w_i$ departs from $w_j$ in the training process. Otherwise, if $w_i$ is similar to $w_j$, then $\text{softmax}(d_j) < \frac{1}{N-1}$ and $w_i$ approaches $w_j$.
\end{remark}

\section{DiversiFed: Convergence Analysis}
In this section, we analyze the convergence properties of DiversiFed. The incremental proximal optimization framework in section \ref{Incremental Optimization Method} is commonly used. Previous work \cite{huang2021personalized} has shown that the loss can converge when the loss function is convex or nonconvex. However, there has not been a demonstration of the gap between the solution obtained using the incremental proximal optimization framework and the solution of the original problem. For completeness, in this paper, we aim to further demonstrate that the solution obtained DiversiFed can converge to the optimal solution of the original problem when the loss function is strongly convex. Before proceeding, we provide the overall optimization objective of DiversiFed.

According to the training objective of each client defined in Eq.\eqref{loss function}, the optimization objective of DiversiFed is
\begin{linenomath}
\begin{equation}\label{overall loss function}
	L(W) = L_{e}(W) + \lambda  L_{d}(W),
\end{equation}
\end{linenomath}
where $W=[w_1, w_2, ..., w_N]$ is a matrix whose columns are all the personalized models, $L_{e}(W)=\sum_{i}L_{e}(w_i)$ is the sum of loss function on the local data and $L_{d}(W)=\sum_{i}L_{d}(w_i)$ is the sum of model distance loss defined in Eq.~\eqref{contrastive loss}.

As we described in section~\ref{Incremental Optimization Method}, we first optimize 
\begin{linenomath}
\begin{equation}\label{overall server optimize}
	Z^{t} = W^{t} - \alpha_t \nabla_{W^{t}} L_{d}(W^{t})
\end{equation}
\end{linenomath}
on the server-side. After getting $Z^t$, we optimize 
\begin{linenomath}
\begin{equation}\label{overall client optimize}
	W^{t+1} = \arg \min_{W} \{L_{e}(W) + \frac{\lambda}{2\alpha_t}||W - Z^{t}||^2\}
\end{equation}
\end{linenomath}
on the client-side. Next, we give the convergence analysis of DiversiFed when $L$ is strongly convex.

\begin{assumption} \label{bound gradient assump}
    (Bounded gradient). The gradient of $L_{e}$ is bounded with $||\nabla L_{e}|| \le C$ and the gradient of $L_{d}$ is bounded with $||\nabla L_{d}|| \le \frac{C}{\lambda}$, where $C$ is a constant.
\end{assumption}

\begin{assumption} \label{strong convex assump}
    (Strong convexity). $L_{e}$ is strongly convex with parameter $\mu > 0$, i.e., $\forall W, W'$, we have
    \begin{linenomath}
    \begin{equation}
        L_{e}(W) - L_{e}(W') \ge \left<\nabla L_{e}(W'), W-W' \right> + \frac{\mu}{2} ||W - W'||^2. \nonumber
    \end{equation}
    \end{linenomath}
\end{assumption}

\begin{theorem}\label{convextheorem}
    (Convex DiversiFed's convergence). Let Assumption \ref{bound gradient assump} and Assumption \ref{strong convex assump} hold, if $L_{d}$ is convex and $\alpha_1 = ... = \alpha_T = \frac{1}{\sqrt{T}}$, then $W^t$ obtained by DiversiFed satisfies
    \begin{linenomath}
    \begin{equation}
        \min_{0 \le t \le T-1} ||W^t - W^*||^2 \le \frac{\lambda}{\mu \sqrt{T}}||W^0-W^*||^2 + \frac{5C^2}{\mu \lambda \sqrt{T}}, \nonumber
    \end{equation}
\end{linenomath}
    where $W^*$ denotes the optimal solution of problem \eqref{overall loss function} and $T > 0$ denotes the maximum communication round.
\end{theorem}
Theorem \ref{convextheorem} illustrates that when $L$ is strongly convex, DiversiFed can converge to the optimal solution of the original problem \eqref{overall loss function} by using the proximal incremental optimization method based on Eq.~\eqref{overall server optimize} and Eq.~\eqref{overall client optimize}, with convergence rate of $\mathcal O(\frac{1}{\sqrt{T}})$. Please refer to the supplemental material for proof.

\section{Experiments}
In this section, we validate the performance of DiversiFed through experiments. First, we compare DiversiFed with two baseline methods and nine state-of-the-art (SOTA) methods, including Per-FedAvg \cite{fallah2020personalized}, pFedMe \cite{t2020personalized}, FedAMP \cite{huang2021personalized}, FedFomo \cite{zhang2020personalized}, APPLE \cite{ijcai2022p301}, \textcolor{black}{FedRep \cite{collins2021exploiting}, FedProto \cite{tan2022fedproto}, FedRoD \cite{chen2022on}, and FedCAC \cite{wu2023bold}}. For a full comparison, we conduct experiments on three benchmark datasets under varying degrees of non-IID settings. We also conduct experiments on a medical dataset to simulate the situation where multiple distributed hospital institutions participate in FL training. Second, we change the value of $\lambda$ in the loss function to verify the effect of the model distance loss on the model training process. Thirdly, we verify the effect of the temperature coefficient $\tau$ in the model distance loss function on personalized models' pulling together and pushing apart effects. Based on this, we verify the relationship between the optimal value of $\tau$ and the degree of non-IID, and give the selection strategy of $\tau$. Fourthly, we verify the robustness of DiversiFed to partial client participation problems. Finally, we illustrate the effect of FL hyperparameters on DiversiFed.

\subsection{Dataset Settings} \label{dataset setting}

\begin{figure}[tb]
	\centering
	\subfloat[$\alpha=0.01$]{
		\includegraphics[width=0.48\linewidth]{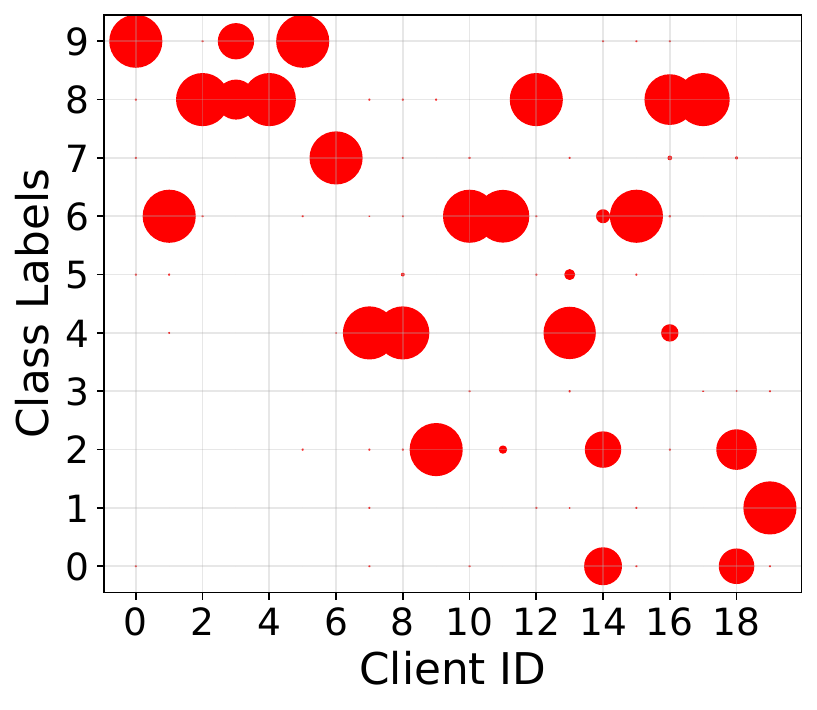}}
	\subfloat[$\alpha=0.1$]{
		\includegraphics[width=0.48\linewidth]{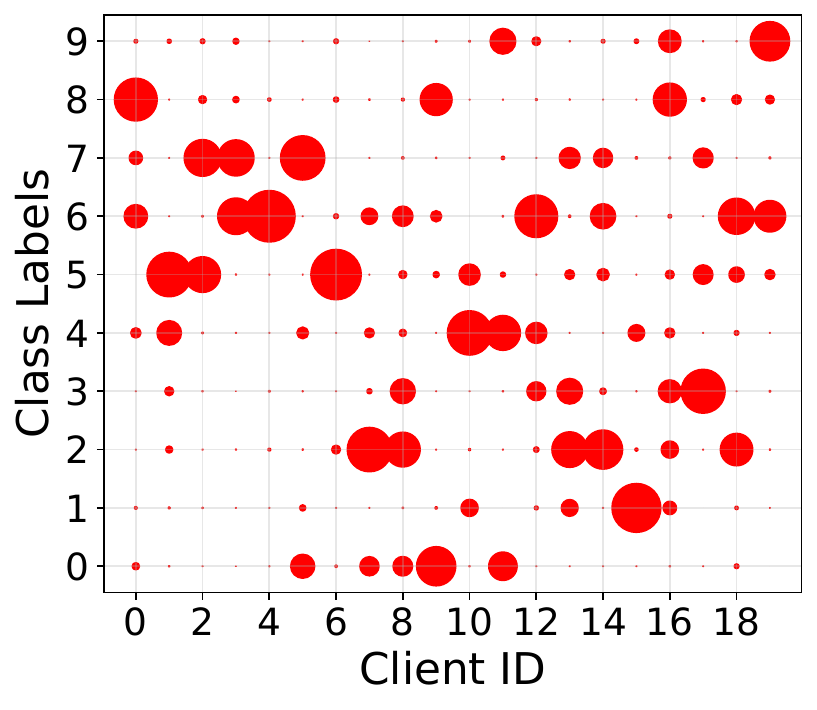}}
    \quad
	\subfloat[$\alpha=0.5$]{
		\includegraphics[width=0.48\linewidth]{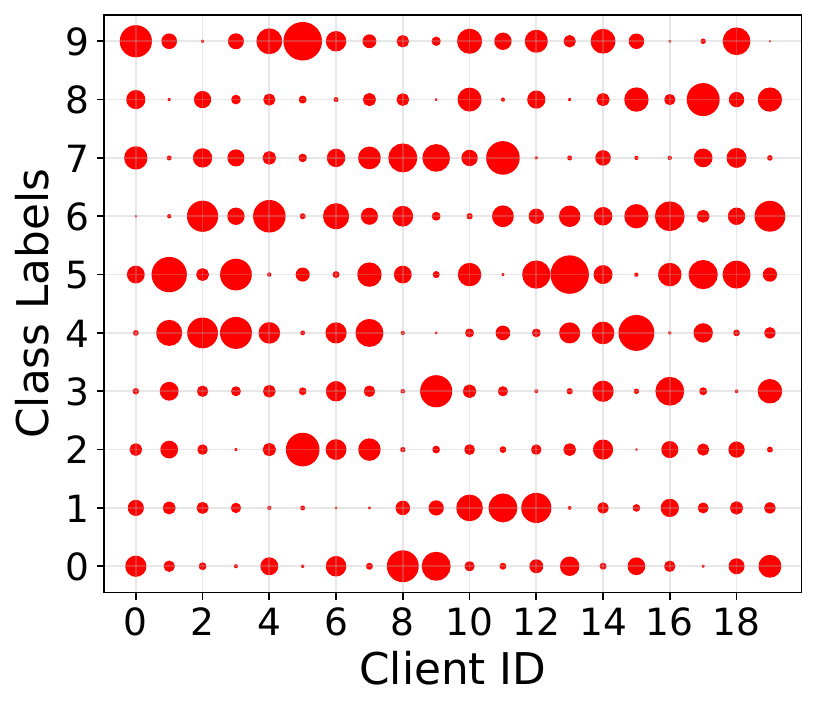}}
	\subfloat[$\alpha=1.0$]{
		\includegraphics[width=0.48\linewidth]{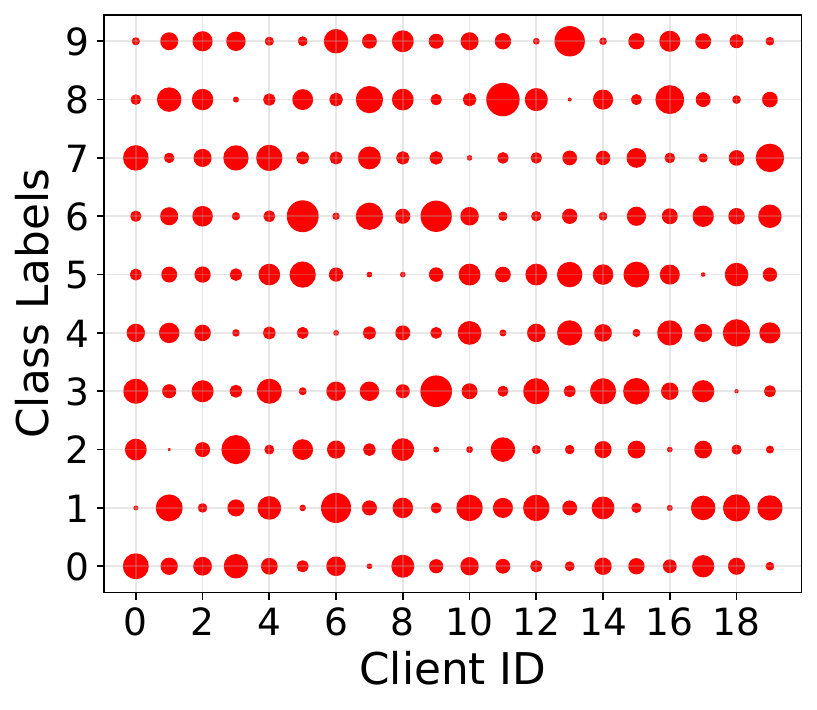}}
	
	\caption{This figure illustrates the data allocation for each client under different $\alpha$ according to Dirichlet distribution. The horizontal axis represents the client ID and the vertical axis represents the data class label index. Red dots represent the data assigned to clients. The larger the dot is, the more data the client has in this class.}
	\label{fig:dirichlet example}
\end{figure}

To thoroughly validate the effectiveness of DiversiFed, we utilize three natural image datasets: FMNIST (Fashion-MNIST) \cite{xiao2017fashion}, CIFAR-10 \cite{krizhevsky2010cifar}, CIFAR-100 \cite{krizhevsky2009learning}, and 9-class colorectal cancer image dataset called PathMNIST \cite{yang2021medmnist} to conduct experiments. To examine the performance of DiversiFed under different non-IID scenarios, for natural image datasets, we employ two commonly used non-IID data partitioning methods to generate datasets for each client. For the medical image dataset, we employ a practical non-IID setup to simulate the scenario where multiple hospitals contribute data.

\textbf{Pathological non-IID:} \ This setting, initially proposed by \cite{mcmahan2017communication}, is widely used in FL research. In this setup, each client is assigned two specific classes of data, with an equal number of samples from each class.

\textbf{Dirichlet non-IID:} \ This is another widely used non-IID setting in the research of FL \cite{hsu2019measuring,lin2020ensemble,kim2022multi,wu2022pfedgf}. In this setting, the data of different clients are drawn $q \sim Dir(\alpha p)$ from Dirichlet distribution, where $p$ represents the prior distribution of all classes, and $\alpha > 0$ is a hyperparameter controlling the identicalness among clients. To facilitate intuitive understanding, in Fig.~\ref{fig:dirichlet example}, we visualize the result of data partitioning in the scenario of the 10-classification task and 20 clients. As the value of $\alpha$ increases, the distribution of clients becomes more similar, indicating a weaker degree of non-IID.

\textbf{Practical non-IID:} \ Due to different regions and diagnostic expertise, different hospitals may have different numbers of samples of different categories for the same disease. Additionally, hospitals within the same region or of similar expertise tend to have more similar data distributions. To simulate this scenario, in this setup, 80\% of each client's data are randomly sampled from the three dominant classes, while the remaining 20\% are randomly sampled from the non-dominant classes. We divided the clients into 3 groups on average, and the clients within each group have the same dominant category. More specifically, the dominant classes of the three groups are $\{0, 1, 2\}$, $\{3,4,5\}$, and $\{6,7,8\}$.

To highlight the effectiveness of FL collaboration, we allocate a small amount of training data to each client. Specifically, for FMNIST, we assign 300 training samples to each client. For CIFAR-10, CIFAR-100, and PathMNIST, we allocate 500 training samples to each client. Additionally, for all datasets, we allocate 100 test samples to each client, with the same label distribution as their respective training data.

\subsection{Comparison Methods and Implementation Details} \label{experiment details}
We compare our method with nine SOTA methods to verify the superiority of our method. They are Per-FedAvg \cite{fallah2020personalized}, pFedMe \cite{t2020personalized}, FedAMP \cite{huang2021personalized}, FedFomo \cite{zhang2020personalized}, APPLE \cite{ijcai2022p301}, \textcolor{black}{FedRep \cite{collins2021exploiting}, FedProto \cite{tan2022fedproto}, FedRoD \cite{chen2022on}, and FedCAC \cite{wu2023bold},} all of which are PFL methods. Per-FedAvg is a meta-learning-based approach. In this approach, all clients work together to train a meta-model first, and then, each client fine-tunes this meta-model to get a personalized model.  
pFedMe is a model-regularization-based method. It constrains personalized models to be closer to their average model. FedAMP, FedFomo, APPLE, and our method can all be regarded as personalized-weight-aggregation-based methods. When aggregating models to generate the personalized model for each client, FedAMP assigns higher weights to models that are similar in parameter space. FedFomo assigns higher weights to those that perform better on the validation dataset. APPLE learns a directed relationship vector to generate weights. Models of clients with similar data distributions are assigned higher weights. \textcolor{black}{FedRep proposes to personalize the classifier and alternately trains the feature extractor and classifier. FedProto exchanges knowledge by aggregating feature centers rather than models. FedRoD introduces two classifiers. A personalized classifier fits the local data distribution, and a shared balanced classifier enables client knowledge exchange. FedCAC enables similar clients to collaborate on parameters that are sensitive to non-IID, while all clients collaborate on parameters that are insensitive to non-IID.}
We also introduce two baseline methods: FedAvg \cite{mcmahan2017communication} and ``Separate". FedAvg is a classic FL algorithm, in which all clients train a global model together. Here, ``Separate" means that all clients are training their models separately without any collaboration. Taking this as a baseline can effectively reflect the effectiveness of different PFL methods to utilize the knowledge of other clients.

For the implementation of the methods, we utilize existing online code repositories as the basis \footnote{Per-FedAvg and pFedMe are implemented with code from \href{github.com/CharlieDinh/pFedMe}{https://github.com/CharlieDinh/pFedMe}. FedAMP, FedFomo, APPLE, FedRep, FedProto, and FedRoD are implemented with code from \href{github.com/TsingZ0/PFL-Non-IID}{https://github.com/TsingZ0/PFL-Non-IID. FedCAC is implemented with code from https://github.com/kxzxvbk/Fling.}}. To ensure fair comparisons, we follow the optimal hyperparameter settings specified in the respective papers for each method. We run all experiments five times and report the mean accuracy along with the standard deviations.

To evaluate the generality of DiversiFed, we conduct experiments using various model structures. On the FMNIST, we utilize a multilayer perceptron with a hidden layer that has 64 units. On the CIFAR-10 and PathMNIST, we utilize a CNN network with two $5 \times 5$ convolutional layers (the first layer has 6 channels, the second layer has 16 channels, and each layer is followed by a $2 \times 2$ max pooling layer) and three fully connected layers each with 400, 120, and 84 units and ReLU activation. On the CIFAR-100, we leverage the ResNet \cite{he2016deep} architecture to experiment. Because of the small amount of client data, it is easy to overfit when using a large deep network. We utilize a relatively small ResNet-8 for the experiments.

In our experiments, unless otherwise specified, we set the total number of distributed clients in FL to 40 (i.e., $N=40$). We run the experiments for 500 communication rounds (i.e., $T=500$) to achieve full convergence. Each client performs local updates for 10 epochs, with a batch size of 100. We evaluate the uniform averaging test accuracy across all clients in each communication round and select the best accuracy as the final result.
On the server-side, we set $\alpha_t=1$. On the client-side, we 
adopt the widely used ADAM \cite{kingma2014adam} optimization algorithm to perform local updates. Here, the local learning rate equals 0.001.

\subsection{Comparing with SOTA on the Pathological non-IID Setting}

\begin{table}[tb]
 \caption{Comparison with state-of-the-art methods under Pathological non-IID.}
	\vskip 0in
	\begin{center}
		\begin{small}
				\begin{tabular}{lcccr}
					\toprule
					Methods & FMNIST & CIFAR-10 & CIFAR-100 \\
					\midrule
					FedAvg & 83.55 $\pm$ 1.16 & 51.19 $\pm$ 1.05 & 23.55 $\pm$ 1.97 \\
					Separate & 96.10 $\pm$ 0.26 & 82.86 $\pm$ 0.73 & 91.45 $\pm$ 0.45 \\
					\midrule
					Per-FedAvg & 95.50 $\pm$ 0.52 & 83.17 $\pm$ 0.69 & 82.30 $\pm$ 0.80 \\
					pFedMe & 96.23 $\pm$ 1.47 & 85.84 $\pm$ 0.78 & 82.52 $\pm$ 1.64 \\
					FedAMP & 96.30 $\pm$ 1.10 & 85.17 $\pm$ 0.83 & 92.16 $\pm$ 0.81 \\
					FedFomo & 96.22 $\pm$ 0.45 & 84.88 $\pm$ 0.35 & 92.38 $\pm$ 0.61 \\
                    APPLE & 95.95 $\pm$ 1.73 & 85.45 $\pm$ 0.92 & 92.23 $\pm$ 0.69 \\
                    \textcolor{black}{FedRep} & \textcolor{black}{96.66 $\pm$ 0.50} & \textcolor{black}{85.07 $\pm$ 1.22} & \textcolor{black}{88.42 $\pm$ 1.21} \\
                    \textcolor{black}{FedProto} & \textcolor{black}{96.24 $\pm$ 0.64} & \textcolor{black}{84.74 $\pm$ 1.02} & \textcolor{black}{92.85 $\pm$ 0.88} \\
                    \textcolor{black}{FedRoD} & \textcolor{black}{\textbf{96.74 $\pm$ 0.46}} & \textcolor{black}{83.37 $\pm$ 1.26} & \textcolor{black}{88.29 $\pm$ 1.37} \\
                    \textcolor{black}{FedCAC} & \textcolor{black}{95.96 $\pm$ 1.74} & \textcolor{black}{84.15 $\pm$ 0.98} & \textcolor{black}{93.02 $\pm$ 0.50} \\
					\midrule
					DiversiFed & 96.47 $\pm$ 0.83 & \textbf{88.04 $\pm$ 0.72} & \textbf{93.19 $\pm$ 0.10} \\
					\bottomrule
				\end{tabular}
		\end{small}
	\end{center}
 \label{pathological noniid}
	\vskip -0.0in
\end{table}

The experimental results in the pathological non-IID scenario are presented in Table \ref{pathological noniid}. In this setting, each client performs a simple binary classification task locally. This can be observed from the results of the ``Separate" baseline, where clients achieve relatively high accuracy even without collaboration. On the simple FMNIST dataset, the accuracy improvement brought by various PFL methods is not significant. In this scenario, DiversiFed achieves competitive results compared to the state-of-the-art (SOTA) methods.

On the challenging CIFAR-10 and CIFAR-100 datasets, obtaining assistance from other clients becomes crucial for performance improvement. Various PFL methods demonstrate substantial accuracy gains compared to the ``Separate" baseline. However, we also observe that in highly non-IID data distributions, seeking help from certain clients can have a negative impact. On the CIFAR-100 dataset, where clients have very different local tasks, Per-FedAvg and pFedMe, which rely on a global model for assistance, suffer from significant performance degradation, even performing worse than the ``Separate" baseline. FedAMP, FedFomo, and APPLE mitigate this issue by selectively learning from clients with similar data distributions, leading to improved model performance. \textcolor{black}{FedRep and FedRoD personalize the classifier to reduce the influence of non-IID, thus obtaining better performance than Per-FedAvg and pFedMe. However, they still rely on a global feature extractor, which is greatly influenced in extreme non-IID scenarios. FedProto exchanges feature centers instead of model parameters, which can reduce the impact of non-IID data and obtain better model performance. FedCAC mitigates the influence of non-IID by allowing parameters of clients that are susceptible to non-IID to collaborate only with clients that have a similar data distribution. At the same time, it allows all clients to collaborate on parameters not easily influenced by non-IID, thus reducing the impact of non-IID while increasing the degree of collaboration.}  DiversiFed further facilitates collaboration among clients with different data distributions, resulting in the best overall model performance.

\subsection{Comparing with SOTA on the Dirichlet non-IID Setting}

We further carry out experiments in the Dirichlet non-IID scenario to evaluate the performance of DiversiFed. We select three values of $\alpha$ to represent different degrees of non-IID and investigate the impact of non-IID degree changes on different methods.

\begin{table*}[tb]
\caption{Comparison with state-of-the-art methods under Dirichlet non-IID On FMNIST, CIFAR-10, and CIFAR-100.}
\begin{center}
\begin{tabular}{@{}c|ccc|ccc|ccc@{}}
\toprule
       & \multicolumn{3}{c|}{FMNIST}              & \multicolumn{3}{c|}{CIFAR-10}              & \multicolumn{3}{c}{CIFAR-100}           \\ \midrule
Method & $\alpha=0.1$ & $\alpha=0.5$ & $\alpha=1.0$ & $\alpha=0.1$ & $\alpha=0.5$ & $\alpha=1.0$ & $\alpha=0.01$ & $\alpha=0.1$ & $\alpha=0.5$ \\ \midrule
FedAvg & 84.53$\pm$1.59 & 84.87$\pm$0.46 & 84.45$\pm$0.81 & 50.31$\pm$2.06 & 51.25$\pm$1.59 & 50.83$\pm$0.62 & 28.97$\pm$1.60 & 35.22$\pm$1.26 & 36.98$\pm$1.06 \\
Separate & 93.37$\pm$1.20 & 85.49$\pm$0.95 & 82.53$\pm$0.74 & 78.94$\pm$1.91 & 55.00$\pm$1.70 & 45.40$\pm$2.09 & 83.74$\pm$0.94 & 46.41$\pm$1.27 & 23.66$\pm$0.36 \\
Per-FedAvg & 92.58$\pm$0.49 & 87.58$\pm$0.47 & 86.11$\pm$0.65 & 79.45$\pm$2.99 & 59.10$\pm$0.83 & 54.31$\pm$0.90 & 63.08$\pm$2.39 & 24.85$\pm$1.31 & 17.53$\pm$0.56 \\
pFedMe & 94.53$\pm$0.83 & 89.33$\pm$0.83 & 87.31$\pm$0.14 & 80.21$\pm$1.01 & 61.01$\pm$1.78 & 54.90$\pm$1.44 & 75.33$\pm$2.70 & 45.02$\pm$1.49 & 32.02$\pm$1.54 \\
FedAMP & 94.74$\pm$0.85 & 89.38$\pm$1.10 & 86.97$\pm$0.50 & 81.29$\pm$0.97 & 63.40$\pm$1.40 & 57.34$\pm$1.19 & 84.47$\pm$0.95 & 52.73$\pm$0.81 & 37.32$\pm$0.31 \\
FedFomo & 94.70$\pm$0.65 & 89.64$\pm$0.55 & 87.26$\pm$1.12 & 80.23$\pm$0.85 & 60.71$\pm$0.75 & 56.01$\pm$1.61 & 84.52$\pm$1.56 & 50.86$\pm$1.24 & 26.05$\pm$0.57 \\
APPLE & 94.73$\pm$0.61 & 89.35$\pm$1.04 & 87.13$\pm$0.93 & 81.61$\pm$0.59 & 63.95$\pm$1.01 & 57.55$\pm$0.76 & 84.69$\pm$0.94 & 52.03$\pm$1.88 & 35.87$\pm$0.91 \\
\textcolor{black}{FedRep} & \textcolor{black}{94.84$\pm$0.72} & \textcolor{black}{89.89$\pm$0.40} & \textcolor{black}{87.11$\pm$0.86} & \textcolor{black}{81.00$\pm$2.42} & \textcolor{black}{59.45$\pm$0.97} & \textcolor{black}{51.71$\pm$1.99} & \textcolor{black}{81.98$\pm$1.53} & \textcolor{black}{52.11$\pm$1.11} & \textcolor{black}{28.32$\pm$1.29} \\
\textcolor{black}{FedProto} & \textcolor{black}{93.65$\pm$1.43} & \textcolor{black}{85.27$\pm$0.76} & \textcolor{black}{81.71$\pm$0.39} & \textcolor{black}{80.77$\pm$3.00} & \textcolor{black}{56.81$\pm$1.48} & \textcolor{black}{49.21$\pm$1.15} & \textcolor{black}{85.39$\pm$0.99} & \textcolor{black}{53.00$\pm$1.55} & \textcolor{black}{27.85$\pm$1.25} \\
\textcolor{black}{FedRoD} & \textcolor{black}{94.46$\pm$0.48} & \textcolor{black}{\textbf{90.55$\pm$0.47}} & \textcolor{black}{87.75$\pm$0.91} & \textcolor{black}{80.67$\pm$2.36} & \textcolor{black}{63.28$\pm$0.27} & \textcolor{black}{56.23$\pm$1.45} & \textcolor{black}{84.26$\pm$1.07} & \textcolor{black}{53.02$\pm$0.64} & \textcolor{black}{37.83$\pm$1.75} \\
\textcolor{black}{FedCAC} & \textcolor{black}{94.87$\pm$0.89} & \textcolor{black}{90.23$\pm$0.40} & \textcolor{black}{\textbf{87.93$\pm$0.96}} & \textcolor{black}{80.32$\pm$4.04} & \textcolor{black}{59.40$\pm$0.56} & \textcolor{black}{55.59$\pm$1.71} & \textcolor{black}{85.79$\pm$0.73} & \textcolor{black}{53.45$\pm$1.16} & \textcolor{black}{37.98$\pm$1.26} \\
\midrule
DiversiFed & \textbf{95.13$\pm$1.37} & 89.20$\pm$0.56 & 86.90$\pm$0.47 & \textbf{83.72$\pm$1.24} & \textbf{65.32$\pm$1.14} & \textbf{58.83$\pm$1.44} & \textbf{86.80$\pm$0.93} & \textbf{56.64$\pm$0.86} & \textbf{40.42$\pm$0.71} \\
\bottomrule
\end{tabular}
\end{center}
\label{dirichlet noniid}
\end{table*}

The experimental results on three datasets are shown in Table \ref{dirichlet noniid}. From the results on FMNIST we can see that, as the degree of non-IID gradually decreases (i.e., $\alpha$ gradually increases), the performance of the ``Separate" baseline gradually decreases. This is because as $\alpha$ increases, each client possesses data from more classes, but with fewer samples per class, making the local classification task more challenging for each client. In this scenario, PFL methods demonstrate increasingly significant accuracy improvements compared to the ``Separate" baseline by leveraging knowledge from other clients. Specifically, when the data are highly non-IID (i.e., $\alpha = 0.1$), compared to Per-FedAvg and pFedMe, FedAMP, FedFomo, APPLE, \textcolor{black}{FedCAC,} and DiversiFed mitigate the impact of clients with more diverse data distributions by selectively learning from clients with similar distributions, leading to better model performance. DiversiFed further leverages information from clients with dissimilar data distributions to achieve the best model performance. However, as the degree of non-IID weakens (i.e., $\alpha = 0.5$ and $\alpha = 1$), the benefits of pushing models apart diminish, resulting in slightly lower performance of DiversiFed compared to other PFL methods. Nevertheless, DiversiFed still demonstrates competitive performance. Furthermore, as the degree of non-IID weakens, the advantage of pFedMe and FedRoD becomes more prominent, indicating that when the dataset is simpler and the degree of non-IID is weaker, learning more from other clients by being close in the parameter space can have a positive impact.

The results on CIFAR-10 and CIFAR-100 exhibit similar trends to the FMNIST results as the non-IID degree changes. However, CIFAR-10 and CIFAR-100 are more challenging datasets to learn from, making it even more crucial to effectively leverage the assistance of other clients. In this regard, DiversiFed outperforms other SOTA methods.

Interestingly, on the most challenging CIFAR-100 dataset, \textcolor{black}{especially when $\alpha = 0.5$, FedFomo, FedRep, and FedProto perform much worse than FedAMP, APPLE, FedRoD, FedCAC, and DiversiFed.} We conjecture that this is because FedFomo only chooses to collaborate with personalized models that perform better on its own validation set. In this scenario, due to the difficult classification task, all clients' personalized models perform poorly, making it challenging to find personalized models with superior performance for collaboration. \textcolor{black}{FedProto only aggregates feature centers, and the knowledge exchange among clients is little. FedRep personalizes the classifier so that the classifier cannot be helped by other clients, resulting overfitting of local data.} Moreover, FedAvg achieves higher accuracy than most PFL methods in this scenario. This further indicates that when the local task becomes more difficult, it becomes increasingly necessary for each client to collaborate with different clients to improve the model's generalization. DiversiFed can still achieve good model performance in such scenarios by effectively utilizing different client information.

In general, the experimental results on three datasets show that DiversiFed results in better model performance by learning from clients with diverse distribution, especially when the degree of non-IID is high. At the same time, when the degree of non-IID is weak, the superiority of DiversiFed will decrease, but it still shows competitive results with SOTA methods. We believe this is because when the degree of non-IID is weak, the difference between the optimal personalized models of clients is not so significant. At this time, pushing them apart brings little benefit. Therefore, our method is more suitable for scenarios with a high degree of non-IID.

\subsection{Comparing with SOTA on the Medical Dataset}
\begin{table*}[tb]
\caption{Comparison with state-of-the-art methods under Practical non-IID on PathMNIST.}
	\vskip 0in
	\begin{center}
				\begin{tabular}{ccccccccccccc}
					\toprule
					Methods & FedAvg & Separate & Per-FedAvg & pFedMe & FedAMP & FedFomo & APPLE & \textcolor{black}{FedRep} & \textcolor{black}{FedProto} & \textcolor{black}{FedRoD} & \textcolor{black}{FedCAC} & Ours \\
					\bottomrule
                    Accuracy &\makecell{60.61 \\ $\pm$2.23} & \makecell{68.21 \\ $\pm$0.88} & \makecell{68.46 \\ $\pm$2.10} & \makecell{72.81 \\ $\pm$1.75} & \makecell{75.23 \\ $\pm$0.94} & \makecell{75.14 \\ $\pm$1.42} & \makecell{75.41 \\ $\pm$0.68} & \makecell{72.36 \\ $\pm$2.56} & \makecell{69.84 \\ $\pm$0.77} & \makecell{72.90 \\ $\pm$1.01} & \makecell{75.23 \\ $\pm$1.47} & \makecell{\textbf{76.28} \\ \textbf{$\pm$1.17}} \\
                    \bottomrule
				\end{tabular}
	\end{center}
 \label{medical dataset}
	\vskip -0.0in
\end{table*}
In this experiment, we simulate a scenario of distributed collaborative training models from multiple hospitals. We set $20$ clients (i.e., $N=20$). As depicted in Section \ref{dataset setting}, the clients are divided into $3$ groups with $6$, $6$, and $8$ clients in each group. The experiment results are shown in Table.~\ref{medical dataset}.

In this scenario, since clients can naturally group according to data distribution, the weighted aggregation-based methods are significantly superior to other methods. However, previous weighted aggregation-based methods (i.e., FedAMP, FedFomo, and APPLE) only focus on clients with similar data distribution. Clients in these methods mainly get help from clients within the same group, losing benefits other groups might bring. Our DiversiFed can also benefit from clients in other groups, thus obtaining the best performance.

\subsection{Effects of Hyperparameter $\lambda$}
\begin{figure}[tb]
 \centering
		\centerline{\includegraphics[width=\linewidth]{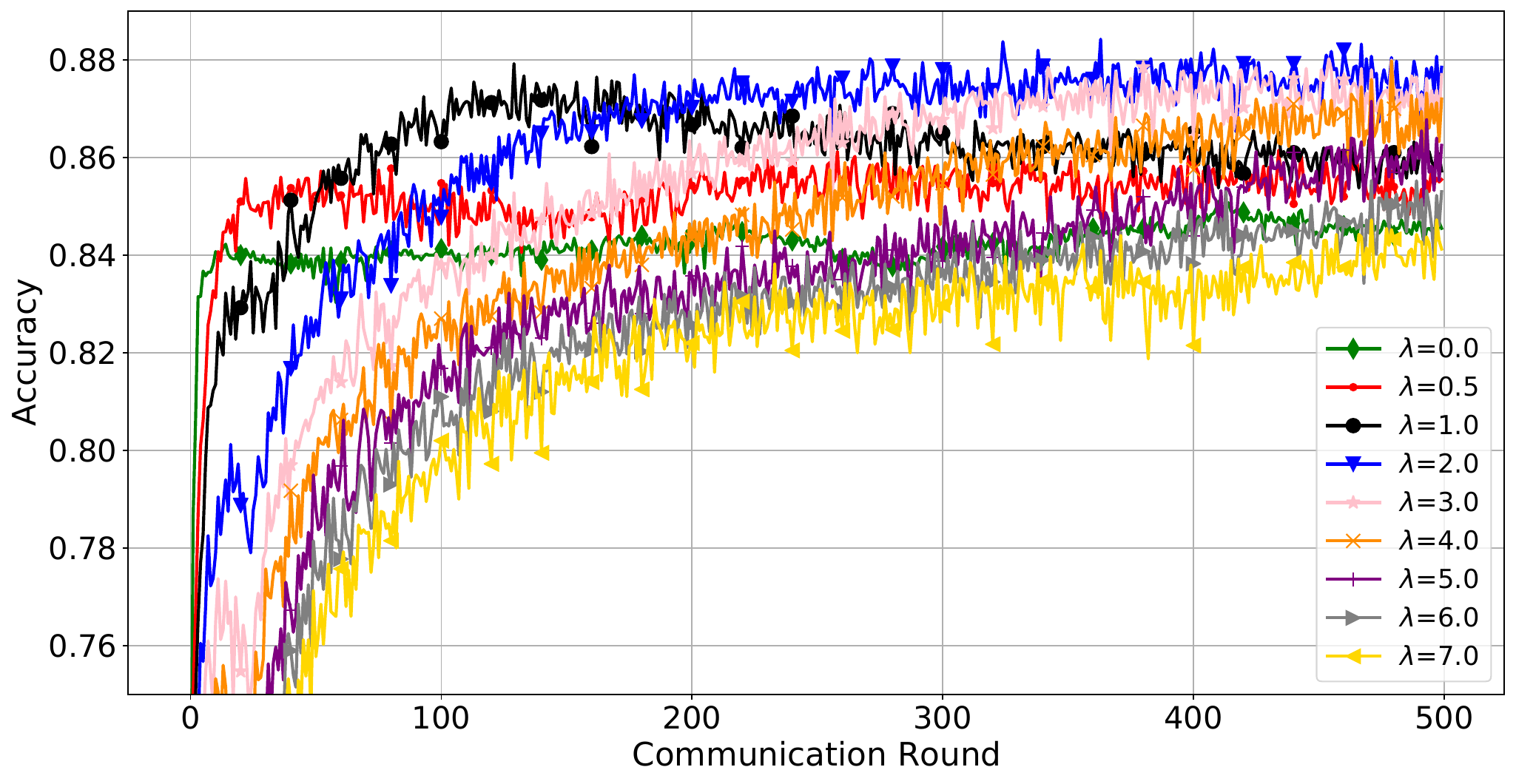}}
 \caption{Effect of $\lambda$ on the training process of DiversiFed on the CIFAR-10.}
		\label{effect of lambda}
\end{figure}

In this section, we investigate the effect of the hyperparameter $\lambda$ in Eq.~\eqref{loss function} on the model training process and final accuracy. We conduct this experiment under the pathological non-IID setting using the CIFAR-10 dataset. We sample $\lambda \in \{0, 0.5, 1, 2, 3, 4, 5, 6, 7\}$.

The experimental results are shown in Fig.~\ref{effect of lambda}. The green line (i.e., $\lambda=0$) serves as the baseline, indicating that clients only use cross-entropy loss. From the red line (i.e., $\lambda=0.5$) and the black line (i.e., $\lambda=1$), we can see that when $\lambda$ is small, the convergence speed of personalized models is fast, and it fully converges in about 150 rounds. However, in this case, the model is easy to overfit, resulting in the low testing accuracy of the final model. On the other hand, we can see that the final accuracy of the model will also be degraded when $\lambda$ is large. This is because the model is overly affected by $L_{d}$ in the training process and ignores characteristics of the local data distribution, resulting in underfitting of the local data and poor performance in the testing dataset. 

From the experimental results, we can conclude that introducing the model distance loss $L_{d}$ can effectively improve the generalization of the models and enhance collaboration among clients. The hyperparameter $\lambda$ controls the extent to which personalized models learn from other clients. In our experiments, we set $\lambda=2$ as it achieved a good balance between collaboration and preserving the characteristics of the local data distribution.

\subsection{Effects of $\tau$ on Model Distance}\label{effect of tau on distance}

\begin{figure}[tb]
	\subfloat[]{
		\label{avg dis sim}
		\includegraphics[width=0.48\linewidth]{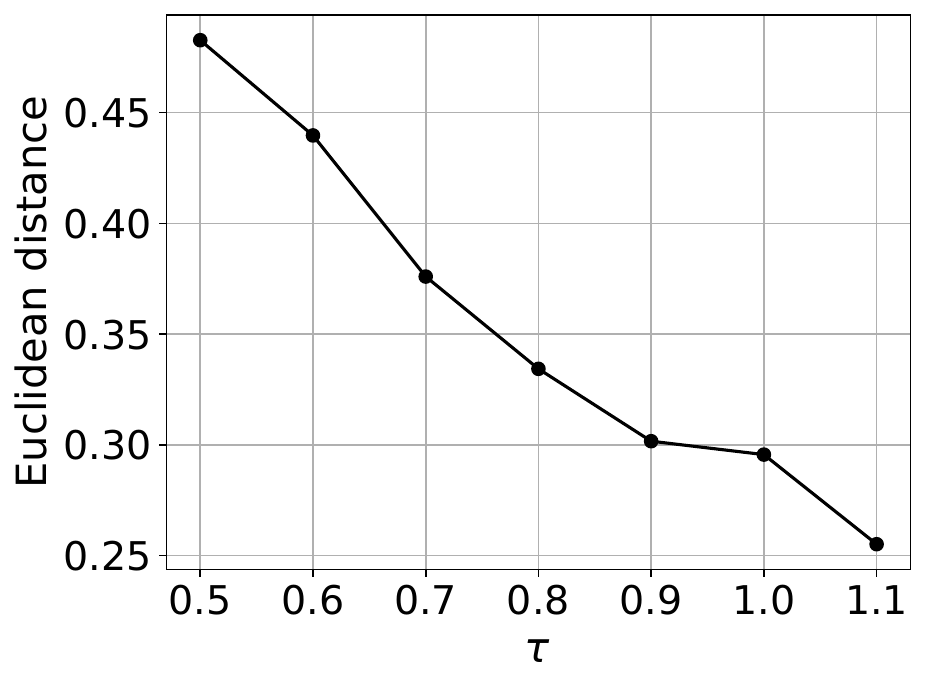}}
	\hspace{-0.0in} 
	\subfloat[]{
		\label{avg dis dissim}
		\includegraphics[width=0.465\linewidth]{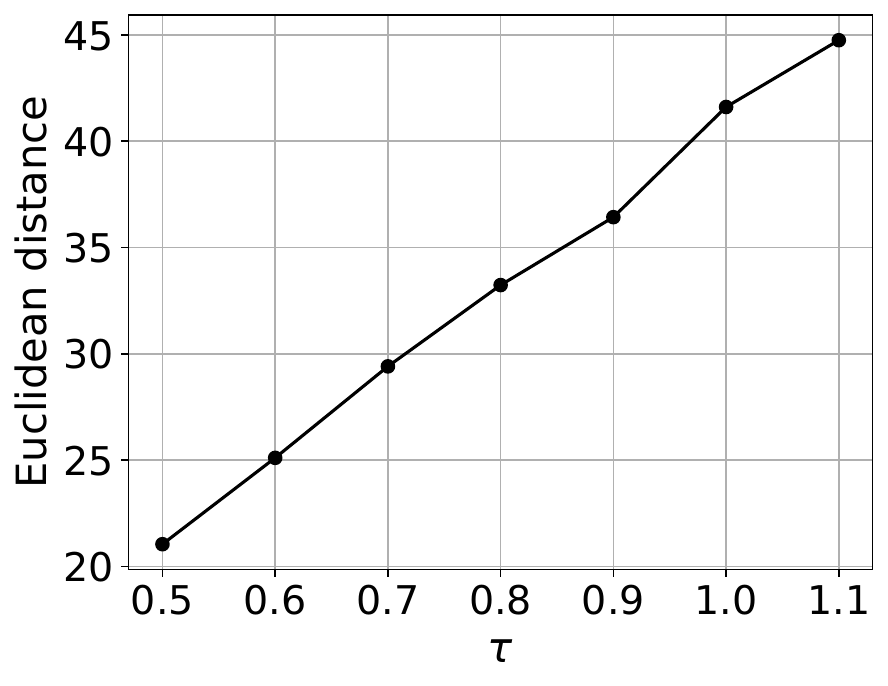}}
	\caption{Experimental results on CIFAR-10 to verify the relationship between $\tau$ and the average model distance. (a) is the average distance between models that are pulled together while (b) is the average distance between models that are pushed apart.}
	\label{fig:2} 
\end{figure}

According to Eq.~\eqref{contrastive loss}, $\tau$ is a hyperparameter used for scaling model distance between personalized models. It affects the ``pulling together'' and ``pushing apart'' effects between models. In this section, we conduct experiments on the CIFAR-10 dataset and pathological non-IID setting to verify the influence of $\tau$ on the distance between personalized models during convergence. We set $\tau \in \{0.5,0.6,0.7,0.8,0.9,1.0,1.1\}$. 

For each $\tau$, when the algorithm converges, we calculate the distance in parameter space between all pairs of personalized models. By examining the order of magnitude of the distance between models, we can easily tell which personalized models are pulled together and which are pushed apart. We calculate the average distance between the personalized models pulled together and those pushed apart.

Fig.~\ref{avg dis sim} and Fig.~\ref{avg dis dissim} show the effects of $\tau$ on the ``pulling together'' and ``pushing apart'' effects of models, respectively. We observe that as $\tau$ increases, the average distance between similar models decreases, indicating a stronger ``pulling together" effect. Conversely, the average distance between dissimilar models increases, reflecting a stronger ``pushing apart" effect. Therefore, the hyperparameter $\tau$ can be adjusted to control the degree of interaction between models in DiversiFed.

\begin{table}[tb]
 \caption{Relationship between $\tau$ and non-IID degree on CIFAR-10.}
	\vskip 0in
	\begin{center}
		\begin{small}
			\begin{sc}
				\begin{tabular}{lcccr}
					\toprule
					& $\alpha=0.1$ & $\alpha=0.5$ & $\alpha=1.0$ \\
					\midrule
					$\tau=0.5$ & 81.80 $\pm$ 3.20 & 64.97 $\pm$ 1.80 & \textbf{59.29 $\pm$ 2.29} \\
					$\tau=0.6$ & 82.05 $\pm$ 2.71 & 65.02 $\pm$ 1.63 & 59.19 $\pm$ 2.79 \\
					$\tau=0.7$ & 82.39 $\pm$ 2.36 & \textbf{65.14 $\pm$ 1.78} & 58.70 $\pm$ 3.10 \\
					$\tau=0.8$ & 82.41 $\pm$ 1.81 & 64.83 $\pm$ 1.79 & 57.94 $\pm$ 2.97 \\
					$\tau=0.9$ & 82.41 $\pm$ 1.77 & 63.29 $\pm$ 2.62 & 56.60 $\pm$ 2.37 \\
					$\tau=1.0$ & \textbf{82.65 $\pm$ 2.06} & 62.02 $\pm$ 3.07 & 53.04 $\pm$ 1.52 \\
					$\tau=1.1$ & 82.62 $\pm$ 1.88 & 61.18 $\pm$ 1.72 & 51.45 $\pm$ 3.10 \\
					\bottomrule
				\end{tabular}
			\end{sc}
		\end{small}
	\end{center}
	\label{tau and alpha on cifar10}
	\vskip -0.0in
\end{table}
\begin{table}[tb]
 \caption{ Relationship between $\tau$ and non-IID degree on FMNIST.}
	\vskip 0in
	\begin{center}
		\begin{small}
			\begin{sc}
				\begin{tabular}{lcccr}
					\toprule
					& $\alpha=0.1$ & $\alpha=0.5$ & $\alpha=1.0$ \\
					\midrule
					$\tau=0.5$ & 95.36 $\pm$ 0.78 & 88.07 $\pm$ 1.64 & \textbf{87.44 $\pm$ 0.23} \\
					$\tau=0.6$ & 95.57 $\pm$ 0.70 & \textbf{88.64 $\pm$ 1.64} & 87.03 $\pm$ 0.67 \\
					$\tau=0.7$ & 95.60 $\pm$ 0.73 & 88.62 $\pm$ 1.57 & 87.04 $\pm$ 0.60 \\
					$\tau=0.8$ & \textbf{95.77 $\pm$ 0.76} & 88.53 $\pm$ 1.26 & 86.86 $\pm$ 0.56 \\
					$\tau=0.9$ & \textbf{95.77 $\pm$ 0.84} & 87.86 $\pm$ 1.29 & 86.75 $\pm$ 1.10 \\
					$\tau=1.0$ & 95.66 $\pm$ 0.76 & 87.55 $\pm$ 1.19 & 86.47 $\pm$ 1.05 \\
					$\tau=1.1$ & 95.58 $\pm$ 0.72 & 87.39 $\pm$ 1.27 & 86.00 $\pm$ 1.57 \\
					\bottomrule
				\end{tabular}
			\end{sc}
		\end{small}
	\end{center}
	\label{tau and alpha on fmnist}
	\vskip -0.0in
\end{table}
\begin{table}[tb]
 \caption{ Relationship between $\tau$ and non-IID degree on CIFAR-100.}
	\vskip 0in
	\begin{center}
		\begin{small}
			\begin{sc}
				\begin{tabular}{lcccr}
					\toprule
					& $\alpha=0.01$ & $\alpha=0.1$ & $\alpha=0.5$ \\
					\midrule
					$\tau=0.5$ & 85.38 $\pm$ 1.42 & 55.99 $\pm$ 1.92 & \textbf{36.93 $\pm$ 0.91} \\
					$\tau=0.6$ & 85.70 $\pm$ 1.54 & \textbf{56.44 $\pm$ 0.70} & 34.80 $\pm$ 0.94 \\
					$\tau=0.7$ & 85.86 $\pm$ 1.36 & 55.29 $\pm$ 0.05 & 33.52 $\pm$ 1.33 \\
					$\tau=0.8$ & 86.09 $\pm$ 1.28 & 54.40 $\pm$ 0.40 & 32.77 $\pm$ 1.46 \\
					$\tau=0.9$ & 86.13 $\pm$ 1.27 & 54.01 $\pm$ 0.42 & 32.87 $\pm$ 1.23 \\
					$\tau=1.0$ & \textbf{86.33 $\pm$ 1.35} & 53.98 $\pm$ 0.12 & 31.57 $\pm$ 1.71 \\
					$\tau=1.1$ & 86.25 $\pm$ 0.88 & 53.48 $\pm$ 0.24 & 31.04 $\pm$ 1.93 \\
					\bottomrule
				\end{tabular}
			\end{sc}
		\end{small}
	\end{center}
	\label{tau and alpha on cifar100}
	\vskip -0.0in
\end{table}

\subsection{Relationship Between $\tau$ and non-IID Degree}\label{relationship tau and noniid}
In the previous section, we verify the influence of $\tau$ on model interaction. Intuitively, the degree of model interaction should be related to the degree of non-IID of data. Therefore, in this section, we verify the relationship between $\tau$ and the degree of non-IID. We conduct experiments on the Dirichlet non-IID setting and select $\alpha \in \{0.1, 0.5, 1.0\}$ as the three non-IID levels on CIFAR-10 for the experiments. 

As shown in Table \ref{tau and alpha on cifar10}, when the degree of non-IID is high (i.e., $\alpha=0.1$), the optimal $\tau$ is close to 1.0. As the degree of non-IID weakens, the optimal value of $\tau$ also decreases. When $\alpha=0.5$, the optimal $\tau$ decreases to about 0.7, and when $\alpha=1$, the optimal $\tau$ is less or equal to 0.5. There is an obvious negative correlation between $\alpha$ and $\tau$. To verify the rule's universality, we also conduct experiments on the FMNIST and CIFAR-100 datasets. Similar trends to CIFAR-10 can be seen in Table \ref{tau and alpha on fmnist} and Table \ref{tau and alpha on cifar100}. The experiments on three datasets show that $\tau$ can make DiversiFed find the optimal degree of collaboration between personalized models in different scenarios to obtain the optimal performance.

Combined with the conclusion in section \ref{effect of tau on distance}, a higher $\tau$ means that the ``pushing apart'' effect of dissimilar models and the ``pulling together'' effect of similar models are stronger. This indicates that when the degree of non-IID is high, personalized models of clients tend to approach models of clients with similar distributions while departing from those with large data distribution differences. This further validates our idea. When clients have very different data distributions, their optimal models should be very different in the parameter space, and it is more helpful to keep their models away from each other in the parameter space.

In general, $\tau$ is a hyperparameter strongly correlated with the degree of non-IID. When the degree of non-IID is high, we tend to choose a large $\tau$, and vice versa.

\subsection{Robustness to Partial Client Participant Problem}
\begin{table*}[tb]
 \caption{Influence of partial client participation on different methods.}
	\vskip 0in
	\begin{center}
		\begin{normalsize}
				\begin{tabular}{lccccr}
					\toprule
					Methods & 100\% & 90\% & 70\% & 50\% \\
					\midrule
					FedAvg & 51.19 $\pm$ 1.05 & 51.48 $\pm$ 0.63 {\small (+0.29)} & 50.87 $\pm$ 0.85 {\small (-0.32)} & 51.34 $\pm$ 0.63 {\small (+0.15)} \\
					Per-FedAvg & 83.17 $\pm$ 0.69 & 83.25 $\pm$ 0.47 {\small (+0.08)} & 83.23 $\pm$ 0.65 {\small (+0.06)} & 83.03 $\pm$ 0.55 {\small (-0.14)}  \\
					pFedMe & 85.84 $\pm$ 0.78 & 85.88 $\pm$ 0.71 {\small (+0.04)} & 86.06 $\pm$ 0.76 {\small (+0.22)} & 85.90 $\pm$ 0.76 {\small (+0.06)} \\
					FedAMP & 85.17 $\pm$ 0.83 & 85.04 $\pm$ 0.59 {\small (-0.13)} & 84.95 $\pm$ 0.99 {\small (-0.22)} & 84.90 $\pm$ 0.64 {\small (-0.27)} \\
					FedFomo & 84.88 $\pm$ 0.35 & 84.48 $\pm$ 0.95 {\small (-0.40)} & 84.75 $\pm$ 0.66 {\small (-0.13)} & 84.68 $\pm$ 0.29 {\small (-0.20)} \\
                    APPLE & 85.45 $\pm$ 0.92 & 85.41 $\pm$ 0.95 {\small (-0.04)} & 85.47 $\pm$ 0.72 {\small (+0.02)} & 84.76 $\pm$ 0.40 {\small (-0.69)} \\
					\midrule
					DiversiFed & \textbf{88.04 $\pm$ 0.72} & \textbf{88.01 $\pm$ 0.78} {\small (-0.03)} & \textbf{87.83 $\pm$ 0.70} {\small (-0.21)} & \textbf{87.45 $\pm$ 0.70} {\small (-0.59)} \\
					\bottomrule
				\end{tabular}
		\end{normalsize}
	\end{center}
 \label{partial client participant}
	\vskip -0.0in
\end{table*}
In some FL scenarios, especially in edge computing scenarios, the participation of clients in each communication round can be affected by various factors, such as unstable communication links or client offline. This leads to the partial client participation problem, where only a portion of clients participates in the FL training in each round. The robustness of FL algorithms to this problem is crucial for real-world deployments. In this section, we evaluate the performance of DiversiFed and other SOTA methods under different degrees of partial client participation. We consider scenarios where 50\%, 70\%, and 90\% of the clients participate in each communication round. The experiments are conducted on the CIFAR-10 dataset under the pathological non-IID setting, and the results are shown in Table \ref{partial client participant}.

As the number of clients involved in training decreases in a round, the accuracy of DiversiFed is slightly affected. With 50\% client participation, DiversiFed's accuracy is reduced by 0.59\% compared to the case of 100\% client participation. This reduction is expected as DiversiFed cannot fully leverage the information from all clients in the model distance loss calculation. Each client loses some of the assistance provided by other clients. However, DiversiFed can still consistently outperforms other SOTA methods in terms of accuracy. This demonstrates the superior performance of DiversiFed. At the same time, even if only half of the clients participate in the training in each round, the performance of DiversiFed does not decrease significantly, which indicates the robustness of DiversiFed to the partial client participation problem.

Furthermore, we observe that FedAMP, FedFomo, APPLE, and DiversiFed are more susceptible to the partial client participation problem compared to Per-FedAvg and pFedMe. This is because FedAMP, FedFomo, APPLE, and DiversiFed employ weighted aggregation methods, where different aggregate weights are assigned to clients based on the difference in their data distributions during model aggregation. When a client with a higher weight does not participate in the training round, clients relying on this client's information experience a greater loss.

\subsection{Effects of FL Hyperparameters}
The number of distributed client $N$ and the client local update epoch $E$ are two important hyperparameters in FL. In the previous experiment, we simply set $N=40$ and $E=10$. In this section, we investigate the effect of these two hyperparameters on the performance of DiversiFed. We carry out experiments on three datasets under Dirichlet non-IID setting with $\alpha=0.1$.

\textbf{The influence of client number $N$.} The number of clients impacts the FL algorithm in two main ways. Increasing $N$ leads to a larger amount of data available for training in the FL group, potentially providing more information and improving performance. However, a larger $N$ also means more diverse data distributions, which can amplify the impact of non-IID data. In this experiment, we sample $N \in \{20, 40, 60, 80, 100 \}$. The results are shown in Fig.~\ref{fig:effect of client number}.

On the FMNIST dataset, the accuracy first increases and then decreases as $N$ increases. Initially, increasing $N$ from 20 to 40 improves the accuracy as more data becomes available, allowing each client to benefit from the knowledge of other clients. However, beyond $N=40$, the influence of non-IID becomes more significant, leading to a decrease in accuracy. Since FMNIST is relatively simple, a good performance can be achieved with a small amount of data, and adding more clients does not necessarily improve the performance. On the CIFAR-10 and CIFAR-100 datasets, which are more challenging, the benefits of having more data outweigh the impact of non-IID, resulting in increased accuracy with larger $N$.

\begin{figure}[tb]
	\centering
	\subfloat[FMNIST]{
		\includegraphics[width=0.31\linewidth]{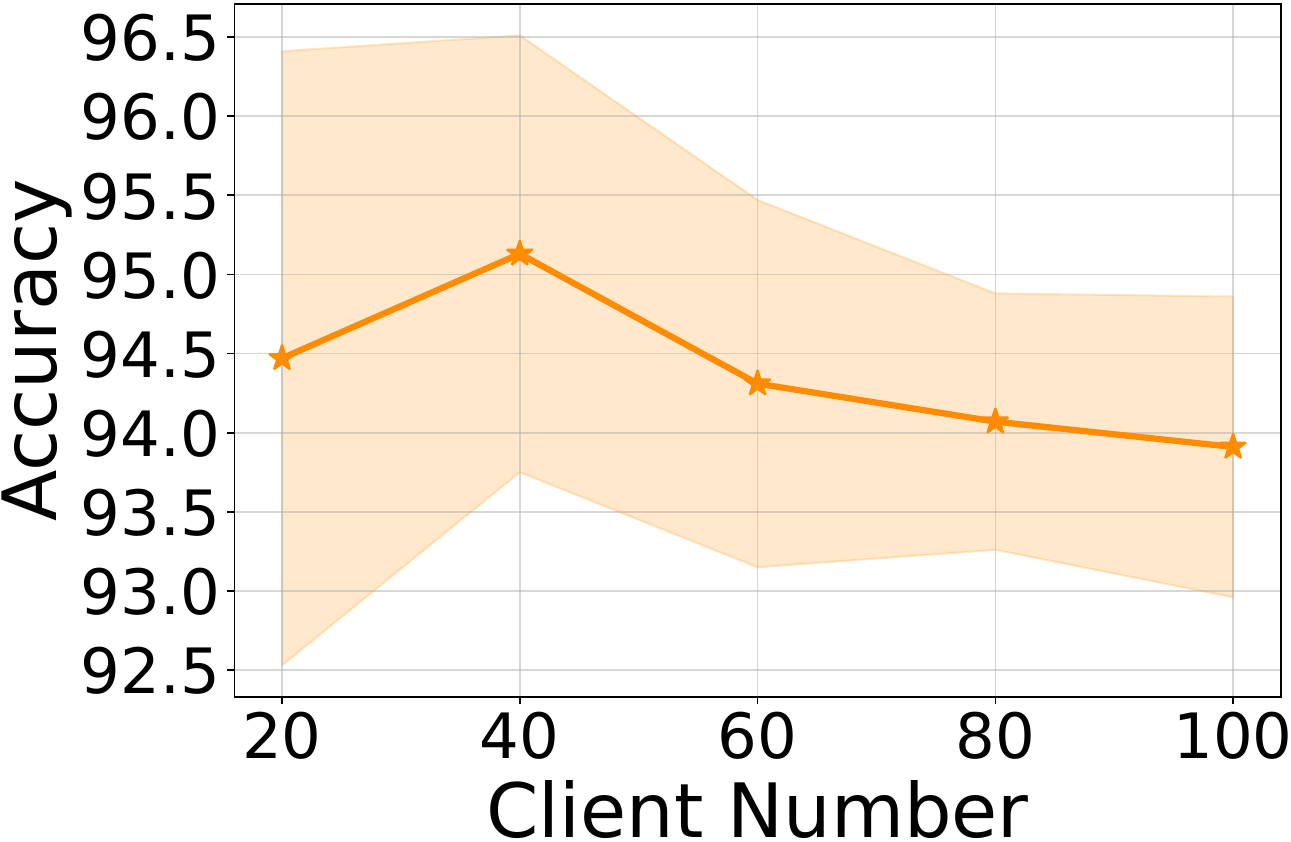}}
	\subfloat[CIFAR-10]{
		\includegraphics[width=0.31\linewidth]{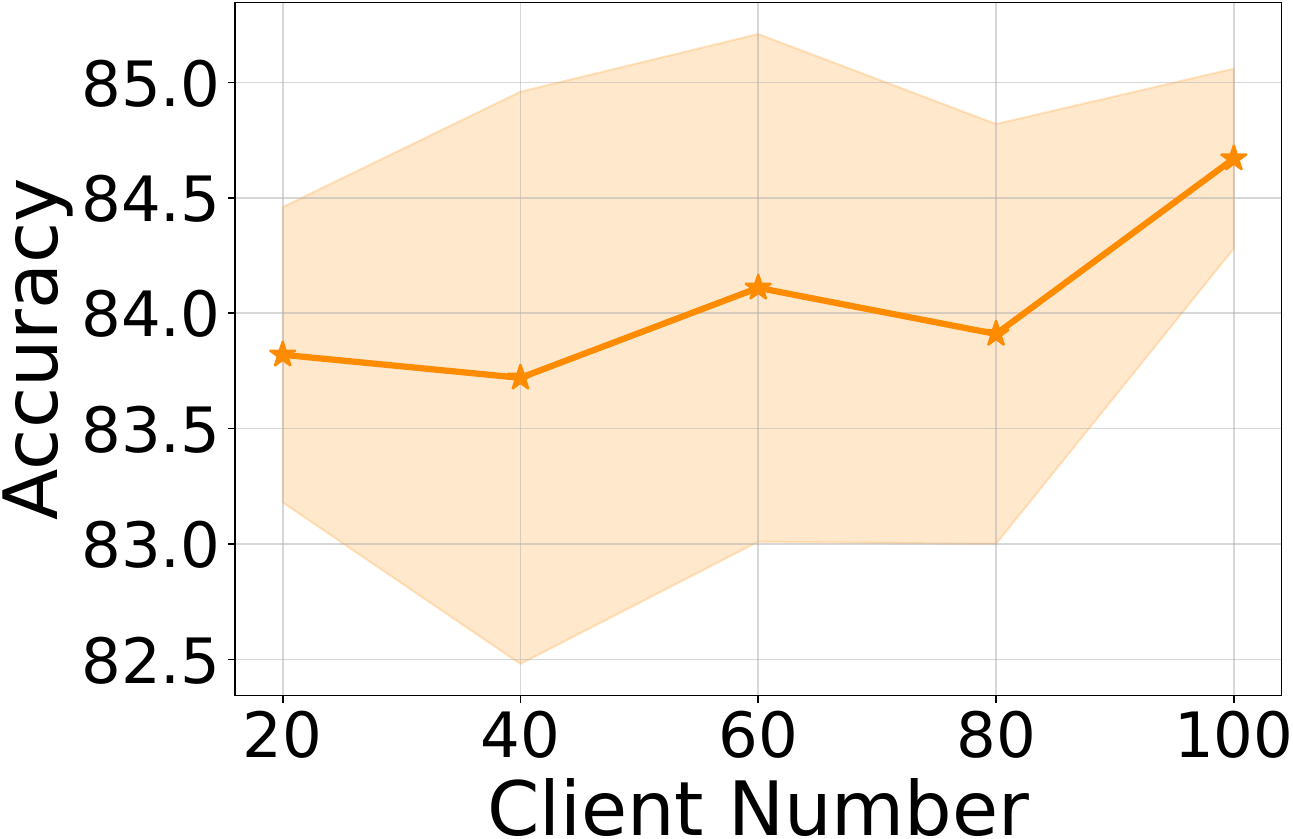}}
	\subfloat[CIFAR-100]{
		\includegraphics[width=0.30\linewidth]{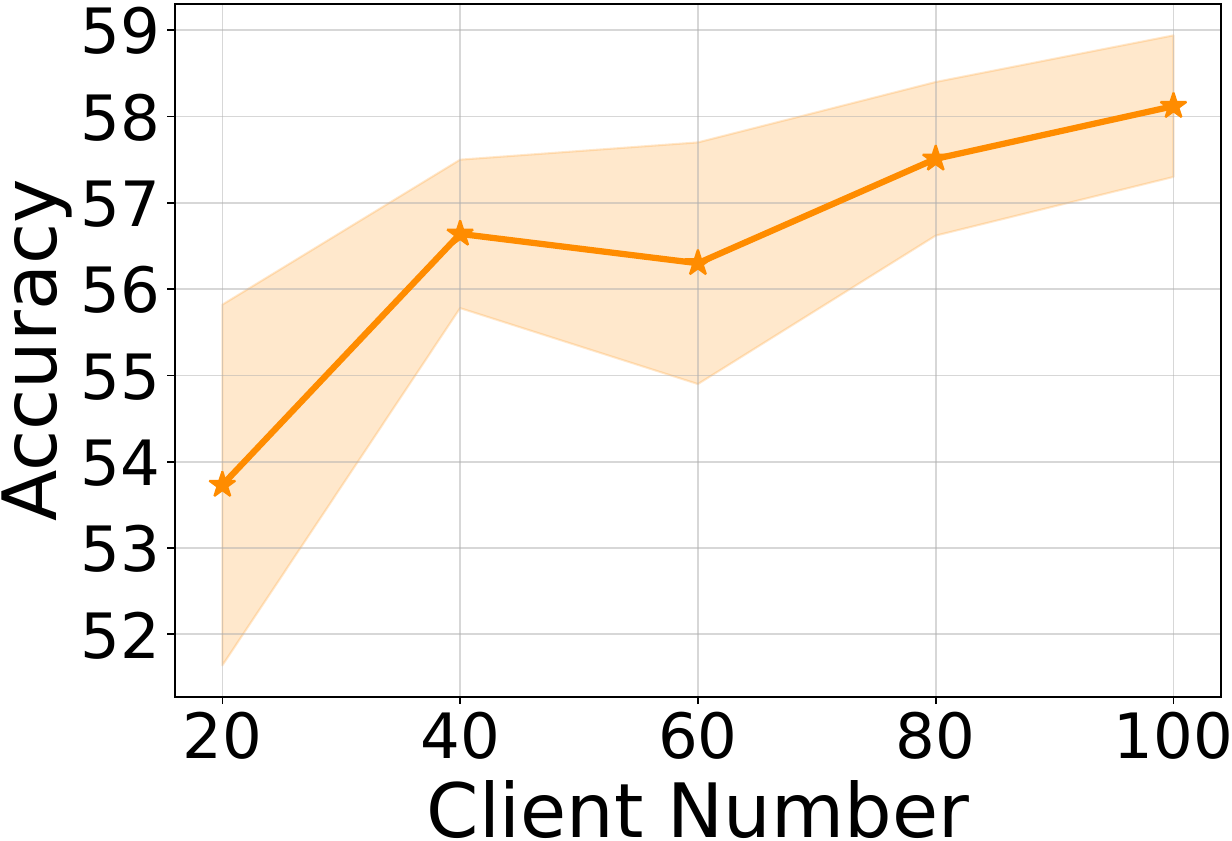}}
	
	\caption{The influence of client number.}
	\label{fig:effect of client number}
\end{figure}

\textbf{The influence of local epoch $E$.} The value of $E$ affects the frequency of model distance loss calculations. A larger $E$ allows clients to learn more from the global model $z_i$, while a smaller $E$ leads to more frequent communication among clients, reducing the impact of client drift due to prolonged local updates. In this experiment, we sample $E \in \{5, 10, 15, 20, 25 \}$. The results are shown in Fig.~\ref{fig:effect of local epoch}.

On the FMNIST and CIFAR-10 datasets, the accuracy of DiversiFed remains relatively stable with varying values of $E$, indicating good robustness of DiversiFed to this hyperparameter. On the CIFAR-100 dataset, there is a significant decrease in accuracy as $E$ becomes larger. This is because the CIFAR-100 dataset has a more diverse data distribution compared to the CIFAR-10 dataset. Consequently, the impact of non-IID becomes more significant, and more frequent communication among clients is needed to mitigate the effects of client drift.

\begin{figure}[tb]
	\centering
	\subfloat[FMNIST]{
		\includegraphics[width=0.31\linewidth]{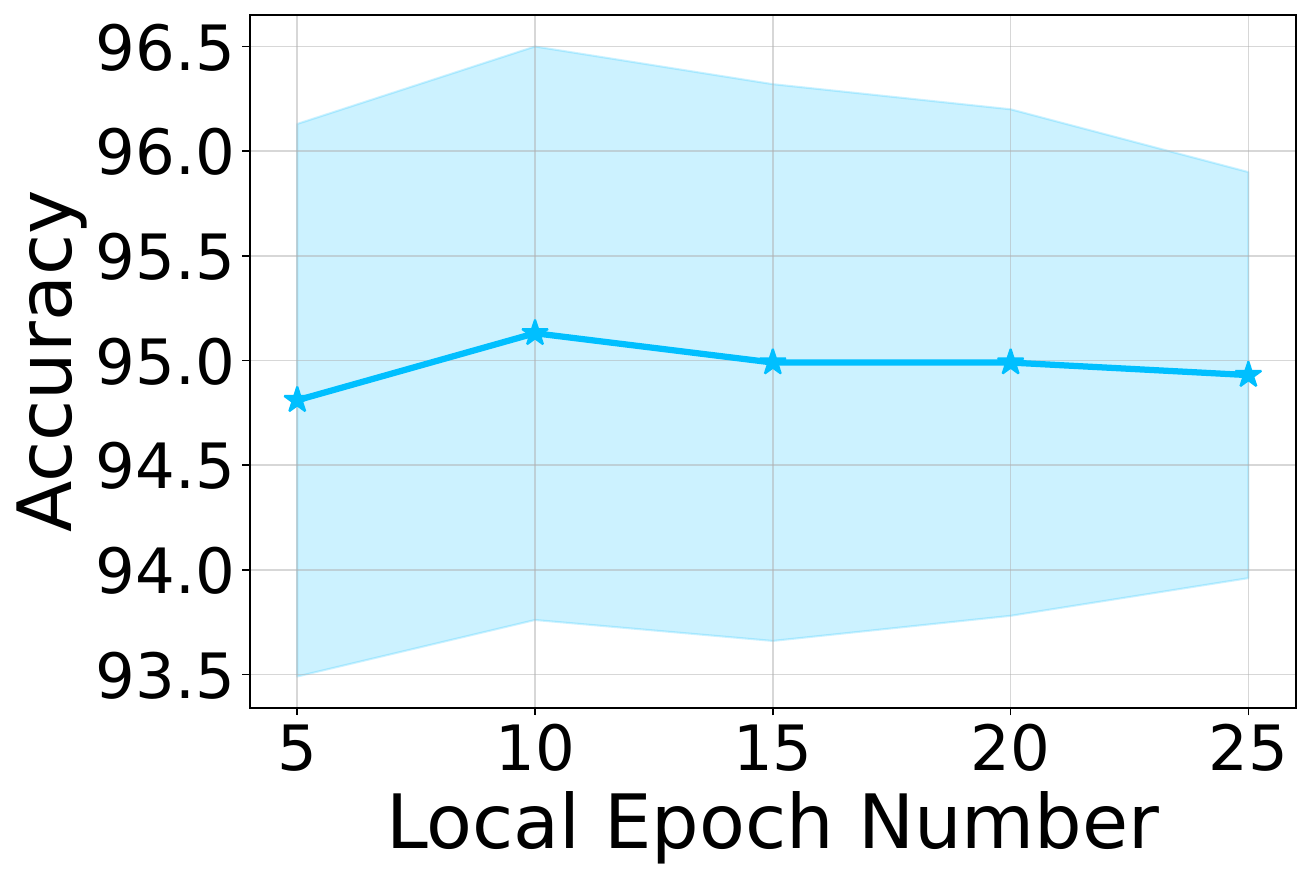}}
	\subfloat[CIFAR-10]{
		\includegraphics[width=0.30\linewidth]{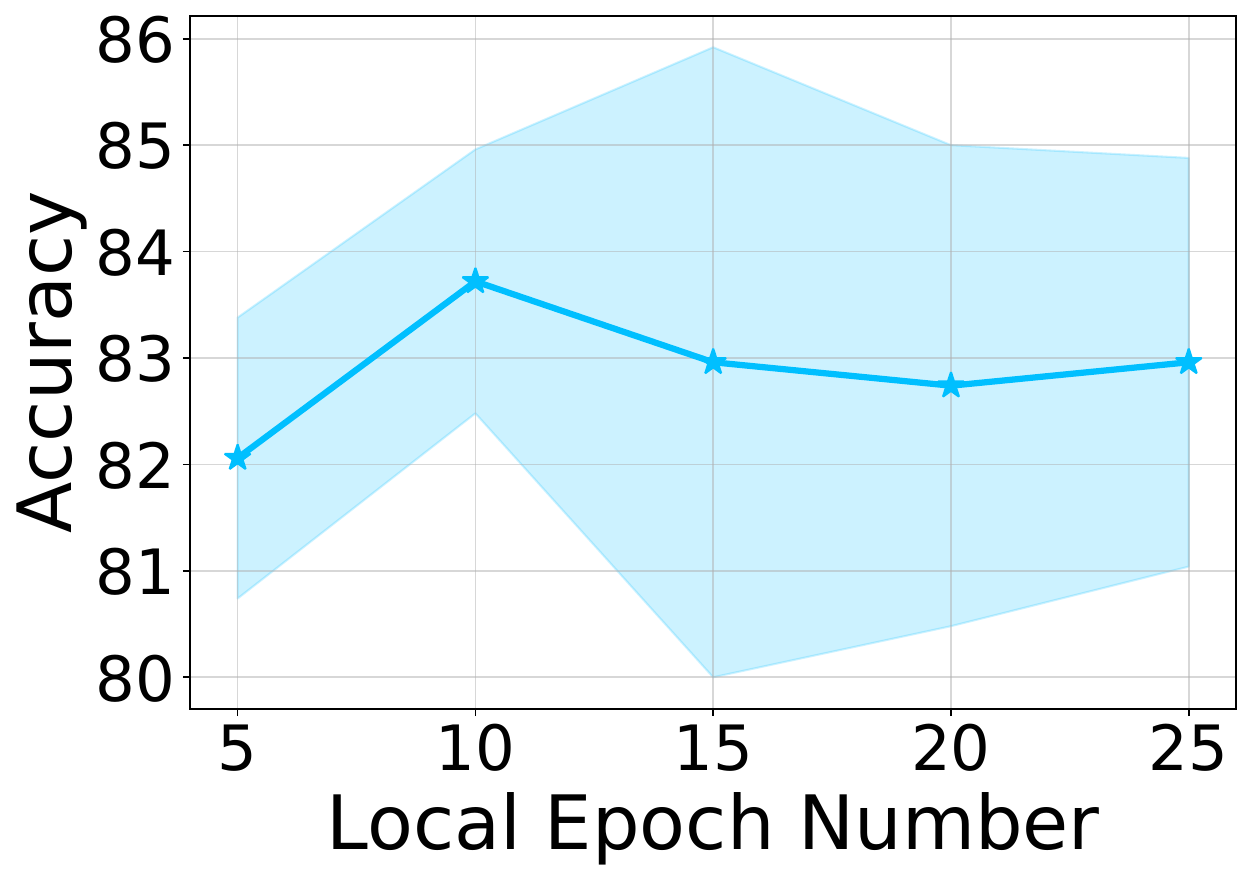}}
	\subfloat[CIFAR-100]{
		\includegraphics[width=0.30\linewidth]{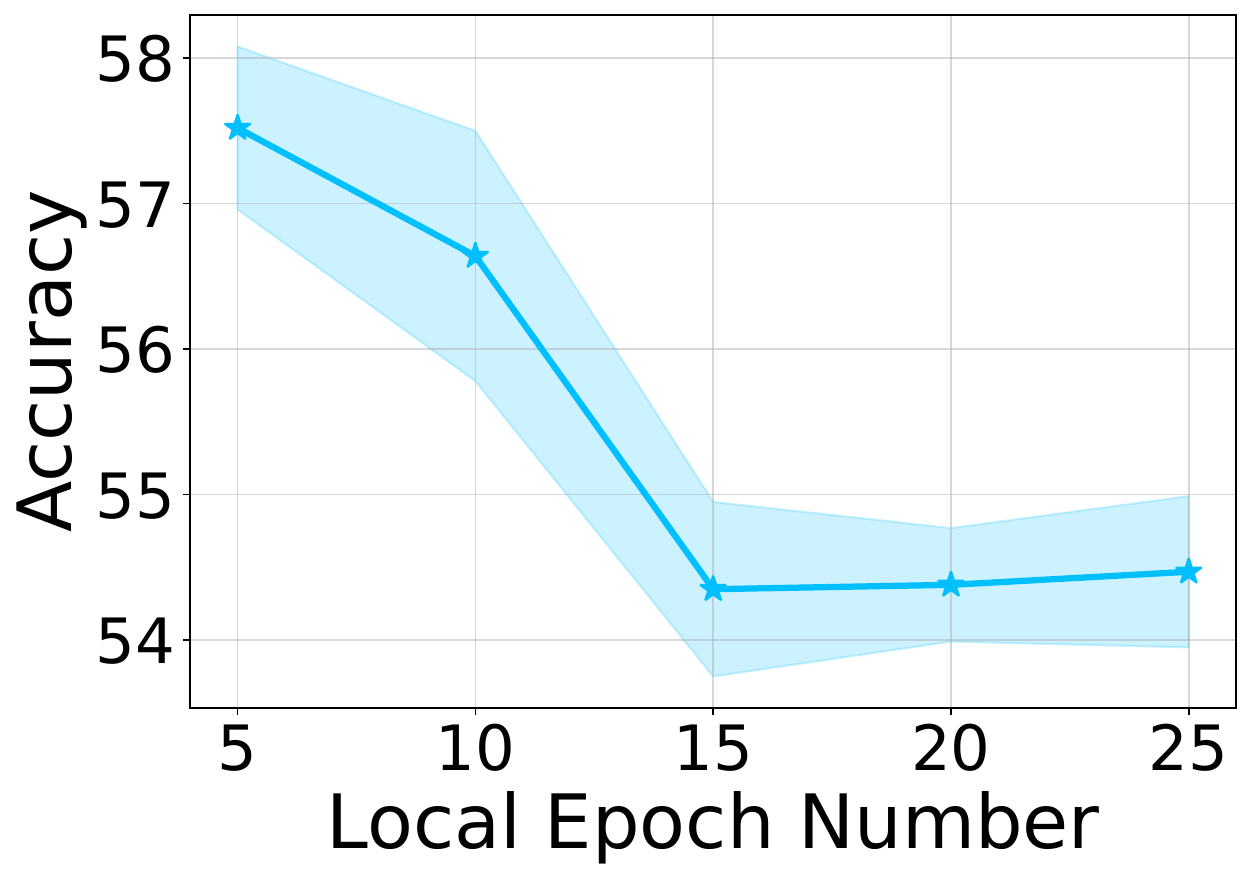}}
	
	\caption{The influence of local epoch number.}
	\label{fig:effect of local epoch}
\end{figure}

\subsection{Training Efficiency Analysis}

\textcolor{black}{In Section~\ref{Incremental Optimization Method}, we propose to utilize an incremental proximal method to minimize the loss function instead of directly optimizing Eq.~\eqref{contrastive loss} on the client side. The benefit of doing so is that it can reduce communication costs while avoiding privacy breaches. However, using this method requires the server to calculate the pair-wise Euclidean distance among clients, which requires quadratic complexity. In this section, we analyze whether this may greatly impede the efficiency of the system.}

We first calculate the computational overhead brought by computing pair-wise model distances on the server side, using flops (floating-point operations) as the metric. Given that PFL typically occurs in cross-silo scenarios, where the number of clients is usually less than 100 \cite{kairouz2021advances}, we estimate the server flops when the number of clients is 40 and 100.

Additionally, to illustrate the impact of the server on system efficiency, we also need to calculate the computing overhead on the client side. Considering that the clients are trained in parallel, we measure the computational overhead of one client within the system. Specifically, for each client, we measure its flops for both forward and backward propagation. For the forward propagation, We adopt the \textit{DeepSpeed Flops Profiler} \cite{deepspeed} to assess the flops. For backward propagation, although there is currently no precise method to calculate, it is widely accepted that backpropagation has twice as many flops as forward propagation \cite{epoch2021backwardforwardFLOPratio, zhou2021efficient}.

\begin{table*}[tb]
\caption{The number of floating point operations (flops) in the primary computational workload of the server and client in one round.}
	\vskip 0in
 \footnotesize
	\begin{center}
				\begin{tabular}{ccccccc}
					\toprule
					Dataset & Model & Params & Client flops & Clients & Server flops & Server / Client ratio (\%) \\
                    \bottomrule
                    \multirow{2}{*}{FMNIST} & \multirow{2}{*}{MLP} & \multirow{2}{*}{50.89\textit{K}} & \multirow{2}{*}{1.53\textit{G}} & 40 & 122.14\textit{M} & 7.98\% \\
                    & & & & 100 & 763.35\textit{M} & 49.89\%  \\
					\bottomrule
                    \multirow{2}{*}{CIFAR-10} & \multirow{2}{*}{CNN} & \multirow{2}{*}{62.01\textit{K}} & \multirow{2}{*}{13.42\textit{G}} & 40 & 148.82\textit{M} & 1.11\% \\
                    & & & & 100 & 930.15\textit{M} & 6.93\%  \\
                    \bottomrule
                    \multirow{2}{*}{CIFAR-100} & \multirow{2}{*}{ResNet-8} & \multirow{2}{*}{1.26\textit{M}} & \multirow{2}{*}{438.24\textit{G}} & 40 & 3.02\textit{G} & 0.69\% \\
                    & & & & 100 & 18.90\textit{G} & 4.31\%  \\
                    \bottomrule
				\end{tabular}
	\end{center}
 \label{flops}
	\vskip -0.0in
\end{table*}

\textcolor{black}{The statistical results of one client and server flops are presented in Table~\ref{flops}. It is evident that during one training round, the computing overhead of the system predominantly resides on the client side. The proportion of computational cost attributed to calculating the Euclidean distance is relatively small. Furthermore, considering the server's substantial computing power (e.g., 1248 TFLOPS (tera floating-point operations per second) with NVIDIA 8x A100 GPUs), which significantly surpasses that of clients (e.g., 36 TFLOPS with an NVIDIA RTX 3090), we conclude that computing pair-wise distances among all clients would not significantly impede the efficiency of the system.}

\section{Conclusion}
In this paper, we propose DiversiFed, a novel approach that addresses the non-IID problem in federated learning by introducing model distance loss at the parameter level. In contrast to previous methods that focus solely on clients with similar data distributions, DiversiFed enables clients to benefit from those with dissimilar data distributions by pushing their models apart in the parameter space. Extensive experimental results on both natural images and medical images demonstrate its superior performance compared to existing methods. DiversiFed has the potential to significantly enhance the effectiveness of federated learning in real-world scenarios with non-IID data.

\bibliographystyle{IEEEtran}
\bibliography{bare_jrnl_new_sample4}

\vfill

\end{document}